\pdfoutput=1

\documentclass[11pt]{article}
\usepackage{bbm}
\usepackage{eqnarray,amsmath,amsfonts,amsthm,mathrsfs}
\usepackage{color}
\usepackage{bm}
\usepackage{amssymb}
\usepackage[dvips]{graphicx}
\usepackage{epsfig}
\usepackage{epsf}
\usepackage{float}
\usepackage{subfigure}
\usepackage{amsfonts,amsmath,amsthm,amssymb,graphicx,float,fancyhdr,multirow,hyperref}
\usepackage{booktabs,longtable,authblk}
\usepackage{mathrsfs,hhline}

\usepackage{fancyhdr}
\usepackage[top=2.5cm, bottom=2.5cm, left=3cm, right=3cm]{geometry}
\usepackage{graphicx}
\newtheorem{theorem}{Theorem}
\newtheorem{assumption}{Assumption}
\newtheorem{corollary}{Corollary}
\newtheorem{definition}{Definition}

\newtheorem{lemma}{Lemma}
\newtheorem{proposition}{Proposition}
\newtheorem{remark}{Remark}
\newtheorem{example}{Example}

\allowdisplaybreaks[4]

\def\begeqn{\begin{equation}}
\def\endeqn{\end{equation}}
\def\begth{\begin{theorem}}
\def\endth{\end{theorem}}
\def\begprop{\begin{proposition}}
\def\endprop{\end{proposition}}
\def\begcor{\begin{corollary}}
\def\endcor{\end{corollary}}
\def\begdef{\begin{definition}}
\def\enddef{\end{definition}}
\def\beglemm{\begin{lemma}}
\def\endlemm{\end{lemma}}
\def\begexm{\begin{example}}
\def\endexm{\end{example}}
\def\begrem{\begin{remark}}
\def\endrem{\end{remark}}
\def\begassum{\begin{assumption}}
\def\endassum{\end{assumption}}

\numberwithin{equation}{section}

\title{Spectral Algorithms on Manifolds through Diffusion
$^\dag$\footnotetext{\dag~The work described in this paper is supported by the National Natural Science Foundation of China (Grants Nos.12171039 and 12061160462) and Shanghai Science and Technology Program [Project No. 21JC1400600]. Email addresses: 22210180107@m.fudan.edu.cn (W. Xia), leishi@fudan.edu.cn (L. Shi).}}

\author{Weichun Xia}
\author{Lei Shi}
\affil{School of Mathematical Sciences and Shanghai Key Laboratory for
	Contemporary Applied Mathematics, Fudan University, Shanghai 200433, China}
\date{}

\begin{document}
	\maketitle
\begin{abstract}

The existing research on spectral algorithms, applied within a Reproducing Kernel Hilbert Space (RKHS), has primarily focused on general kernel functions, often neglecting the inherent structure of the input feature space. Our paper introduces a new perspective, asserting that input data are situated within a low-dimensional manifold embedded in a higher-dimensional Euclidean space. We study the convergence performance of spectral algorithms in the RKHSs, specifically those generated by the heat kernels, known as diffusion spaces. Incorporating the manifold structure of the input, we employ integral operator techniques to derive tight convergence upper bounds concerning generalized norms, which indicates that the estimators converge to the target function in strong sense, entailing the simultaneous convergence of the function itself and its derivatives. These bounds offer two significant advantages: firstly, they are exclusively contingent on the intrinsic dimension of the input manifolds, thereby providing a more focused analysis. Secondly, they enable the efficient derivation of convergence rates for derivatives of any $k$-th order, all of which can be accomplished within the ambit of the same spectral algorithms. Furthermore, we establish minimax lower bounds to demonstrate the asymptotic optimality of these conclusions in specific contexts. Our study confirms that the spectral algorithms are practically significant in the broader context of high-dimensional approximation.
\end{abstract}



\section{Introduction}\label{section: introduction}

Suppose that we are given a dataset $D:=\{(x_i,y_i)\}_{i=1}^n$ independently and identically distributed from an unknown distribution $P$ defined on the product space $X\times Y$ encompassing both the input space $X$ and the output space $Y$. Our goal is to utilize this dataset to approximate the regression function, denoted as $f^*:X\to Y$, which minimizes the mean-squared error
\begin{equation*}
	\int_{X\times Y}\left(y-f(x)\right)^2dP(x,y)
\end{equation*}
over all $\nu$-measurable functions $f:X\to Y$ with $\nu$ representing the marginal distribution of $P$ on $X$.

To address this non-parametric regression problem, a widely adopted  strategy is kernel-based regularization methods. This approach has been extensively explored and validated in many studies, e.g., \cite{smale2007learning,caponnetto2007optimal,yao2007early,steinwart2008support,lin2018distributed,guo2019optimal}. In this paper, we focus specifically on the spectral algorithms, which can be effectively implemented by acting a filter function on the spectra of a finite-dimensional kernel matrix. These algorithms demonstrate efficient performance, primarily attributed to the regularization imposed through a select group of filter functions \cite{bauer2007regularization,gerfo2008spectral}. The choice of the regularization family is pivotal, as it enables the spectral algorithm to cover an extensive range of widely employed regression algorithms, including kernel ridge regression (also referred to as Tikhonov regularization), kernel principal component regression (often known as spectral cut-off), and various gradient methods. Such versatility is exemplified in the literature, see, e.g., \cite{guo2017learning,dicker2017kernel,blanchard2018optimal,mavzcke2018parallelizing,lin2020optimala,celisse2021analyzing}, which offers further insights into the applications of these methodologies. This adaptability also underscores the importance of spectral algorithms in regression analysis.

In many critical areas of modern data analysis, such as image recognition and DNA analysis, the prevalence of high-dimensional data is a notable challenge. This scenario is characterized by the dimensionality of the data significantly surpassing the number of available samples. To address this, researchers frequently turn to various dimensionality reduction techniques as a preliminary step. Among these techniques, some are based on the premise that data inherently resides on low-dimensional feature manifolds, where the intrinsic dimension is markedly lower than the ambient dimension. This concept is supported by evidence in various studies, such as \cite{roweis2000nonlinear,belkin2003laplacian,coifman2006diffusion,van2008visualizing}. Building on this premise, it is plausible to hypothesize that the input space $X$ is a low-dimensional manifold embedded within a higher-dimensional Euclidean space $\mathbb{R}^d$. This suggests that the intrinsic dimension $m$ of the input manifold, now denoted as $\mathcal{M}$, is substantially lower than the ambient dimension $d$. Such a hypothesis aligns with the commonly accepted low-dimensional manifold hypothesis prevalent in related research, as illustrated in studies like \cite{guhaniyogi2016compressed,mcrae2020sample,hamm2021adaptive}. This understanding forms a foundational aspect of our approach, providing an essential prior for the methodologies and techniques employed in our analysis.

Prior research on spectral algorithms within the framework of Reproducing Kernel Hilbert Space (RKHS) has predominantly concentrated on the utilization of general kernel functions. This approach frequently overlooks the intrinsic structure of the input feature space, which is a critical aspect of comprehending and enhancing the performance of algorithms. In our study, we particularly choose to utilize the specific heat kernel $H_t$ with a fixed $t>0$ on the input manifold $\mathcal{M}$ as the generating kernel function for the RKHS. Our objective is to leverage intrinsic properties of the manifold, such as curvature, to enhance the performance of spectral algorithms. This approach represents a novel angle, one that has been relatively unexplored in existing studies on manifold regression. A particularly compelling theoretical challenge in this regression problem is how to derive the convergence rates in the ``hard learning" scenarios, also known as model misspecifications \cite{rao1971some,pillaud2018statistical}. In this situation, the regression function may not necessarily be contained within the RKHS traditionally relied upon. Our focus is on comprehending the generalization error, represented by metrics like the $\alpha$-power norm (see \eqref{Definition of Alpha-Power Norm})
\begin{equation}\label{alpha-power norm in Introduction}
	\|f^*-f_{D,\lambda}\|_{\alpha}
\end{equation}
in these demanding contexts, where $f_{D,\lambda}$ is given by \eqref{Definition of Spectal Alogrithm Estimator}, denoting the estimator of spectral algorithm with regularization parameter $\lambda$.  This aspect of our research is aligned with and generalizes several previous studies \cite{fischer2020sobolev,liu2022statistical,zhang2023optimality}, seeking to provide a deeper understanding and more effective solutions for this hard learning scenario.

Our foremost contribution is the derivation of convergence rates for \eqref{alpha-power norm in Introduction} in hard learning scenarios, where the regression function is defined as $f^*\in \mathcal{H}_t^\beta$ with $\beta<1$ and $\mathcal{H}_t^\beta$ is given by \eqref{Definition of Alpha-Power Space}. We achieve this result by leveraging the rapid decay property of the eigenvalues of the heat kernel, which demonstrates exponential decay attributable to the structure of the input manifold. This attribute facilitates a convergence rate as described in
\begin{equation}\label{Learning rates in Introduction}
	\|f_{D,\lambda_n}-f^*\|^2_{\alpha}\lesssim\left(\frac{(\log n)^\frac{m}{2}}{n}\right)^\frac{\beta-\alpha}{\beta}.
\end{equation} 
In this context, $ A\lesssim B $ means that there exists a constant $ C>0 $ such that $ A\leq CB $. This finding is outlined in Theorem \ref{Theorem: Upper Bound}. Notably, these results are predicated exclusively on the intrinsic dimension $m$ of the input manifold $\mathcal{M}$, echoing findings from previous research \cite{mcrae2020sample}. However, our research expands those results to broader scenarios, offering improved convergence rates. Moreover, our work improves the convergence rate results obtained in prior studies that presupposed polynomial decay \cite{fischer2020sobolev,zhang2023optimality}.
As a consequential benefit, owing to the embedding property of the RKHS generated by the heat kernel, also referred to as the diffusion space, we also ascertain $C^k$-convergence rates for any $k$-th order derivatives, given by
\begin{equation}\nonumber
	\|f_{D,\lambda_n}-f^*\|_{C^k}^2\lesssim\left(\frac{(\log n)^\frac{m}{2}}{n}\right)^{1-\varepsilon}.
\end{equation} 
This is achieved with $\epsilon>0$ selected arbitrarily small, a result that may be intriguing in its own right.

Our second contribution is the demonstration of the optimality of the spectral algorithm using the heat kernel in terms of minimax rates, as elucidated in Theorem \ref{Theorem: Lower Bound}. We specifically focus on the special case when $\alpha=0$, corresponding to the $L^2$-norm error scenario. The minimax lower bound in this situation is given by
\begin{equation}\nonumber
	\|\hat{f}-f^*\|^2_{L^2}\gtrsim\frac{(\log n)^\frac{m}{2}}{n}.
\end{equation} In this context, our derived upper bound successfully aligns with the minimax lower bound. This alignment showcases the effectiveness of the spectral algorithm in this specific setting, underlining its adaptability and efficiency in handling $L^2$-norm error scenarios within the framework of minimax rates. Moreover, it extends previous research on lower bound results for exponential-type decaying eigenvalues (see, e.g., \cite{dicker2017kernel}) to manifold settings with dimension $m\geq2$.
However, in the more general error case $\alpha>0$, we encounter challenges due to the complex nature of the exponential function. One may refer to Section \ref{section: main results} for details. As a result, our method cannot derive a lower bound that aligns precisely with the upper bound indicated by \eqref{Learning rates in Introduction}. Instead, we establish a lower bound that is slightly inferior in the power term of $n$, a finding also evident in Theorem \ref{Theorem: Lower Bound}. Given the potential interest this phenomenon might attract from future researchers, we deliberately underscore it, anticipating that subsequent studies might tackle this challenge in a more refined manner. We will also consider this problem in our future study.

As our third contribution, to the best of our knowledge, our work represents the first exploration into the use of a specific kernel function (the heat kernel) in spectral algorithms for misspecified models. This exploration utilizes a methodology that combines integral operator techniques with the embedding property. Integral operator techniques were initially employed in the study of kernel ridge regression (see, e.g., \cite{caponnetto2007optimal,smale2007learning}) and have recently been further developed for scalable kernel-based methods on massive  data \cite{guo2017learning,mavzcke2018parallelizing,shi2019distributed,lin2020optimala}. However, investigation on integral operator techniques when the target function does not belong to the corresponding reproducing kernel Hilbert space and has low regularity has been scarce until recent work \cite{fischer2020sobolev}, which utilized integral operator techniques in conjunction with the embedding property to investigate the optimal convergence of kernel ridge regression in hard learning scenarios. However, the existing work has focused on using integral operator methods to study a general kernel function. In contrast to several prior works on spectral algorithms for misspecified models, which predominantly employ general polynomial decaying kernels, our investigation focuses on the heat kernels on manifolds. This particular kernel possesses superior eigenvalue decay properties (exponential decay) and nice space embedding properties, reflecting the intrinsic properties of the manifolds and thus leading to improved convergence rates and minimax rates. Notably, our convergence rate is optimal even when the target function exhibits low regularity. This improvement may encourage future research to explore specific kernels with superior performance.

The rest of this paper is structured as follows. In Section \ref{section: preliminaries}, we elucidate the fundamental properties of the heat kernel, diffusion spaces, integral operators, and spectral algorithms. In Section \ref{section: main results}, we present our primary results concerning convergence rates and the minimax lower bounds, alongside some direct corollaries. In Section \ref{Section: Upper Bounds Analysis}, we provide comprehensive proofs of our upper bounds results, which take into account the approximation error in Subsection \ref{Subsection: Approximation Error} and the estimation error in Subsection \ref{Subsection: Estimation Error} with a final proof for Theorem \ref{Theorem: Upper Bound} furnished in Subsection \ref{Subsection: Proof of Theorem1}. We ultimately complete the proof of Theorem \ref{Theorem: Lower Bound} in Section \ref{Section: Lower Bound Analysis}.

\section{Preliminaries}\label{section: preliminaries}

We now turn to a detailed description of our settings. We postulate that the input space $X$ is an $m$-dimensional compact, connected, and immersed submanifold of $\mathbb{R}^d$, which we refer to as $\mathcal{M}$. First, we introduce several fundamental concepts related to the Riemannian manifold $\mathcal{M}$, as discussed in various literatures, including \cite{do2016differential,lee2018introduction}. We specify that the submanifold $\mathcal{M}$ is endowed with the Radon measure $\mu_g$, which is derived from the integral operator on $C_c(\mathcal{M})$, as elucidated by the expression
\begin{equation*}
	\int_{\mathcal{M}}fd_{\mu_g}=\mu_g(f)=
	\sum_j\int_{\varphi_j(U_j)}\left(\psi_jf\sqrt{{\rm det}(g_{ij})}
	\right)\circ\varphi_j^{-1}dx_j^1\cdots dx_j^m.
\end{equation*} 
In this context, $g=(g_{ij})$ signifies the Riemannian metric, $\{U_j,\varphi_j;x_j\}_j$ denotes a smooth atlas, and $\{\psi_j\}_j$ represents the partition of unity subordinate to the covering $\{U_j\}_j$. It is important to note that if $\mathcal{M}$ is also orientable, a smooth, non-zero $m$-form, commonly known as a volume form, can be naturally selected on $\mathcal{M}$ at the local level, as defined by
\begin{equation*}
	\omega=\sqrt{{\rm det}(g_{ij})}dx^1\wedge\cdots\wedge dx^m.
\end{equation*}
This volume form is characterized by the property that its integral over $\mathcal{M}$ is positive, establishing the orientation of $\mathcal{M}$ as per $[\omega]$. Consequently, this leads to the result
\begin{equation*}
	\mu_g(f)=\int_{\mathcal{M}}f\omega.
\end{equation*}

We now focus on the Laplace-Beltrami operator on $\mathcal{M}$, which is defined as 
\begin{equation}\nonumber
	\Delta f = -{\rm div}(\nabla f).
\end{equation}
The formulation of this definition is grounded in $C_c^\infty(\mathcal{M})$ and is further extended to $L^2(\mathcal{M};\mu_g)$ through the application of Friedrich's extension, a detailed discussion of which can be found in \cite{de2021reproducing}. Particularly, in the scenario where $\mathcal{M}=\mathbb{R}^m$, the Laplace-Beltrami operator aligns with the conventional second-order derivatives as indicated in $-\sum_{i=1}^m\partial^2_i$. To facilitate ease of reference in forthcoming discussions, we will henceforth denote $\Delta$  as the Laplacian.

Utilizing the classical Sturm-Liouville decomposition approach for the Laplacian $\Delta$, we are able to delineate an orthonormal basis $\{\varphi_k\}_{k\in\mathbb{N}}$ for $L^2(\mathcal{M};\mu_g)$. This approach and its underpinnings are thoroughly elaborated in \cite{chavel1984eigenvalues}. Within this framework, each $\varphi_k$ is characterized by its smoothness and adheres to the condition
\begin{equation}\nonumber
	\Delta \varphi_k=\lambda_k\varphi_k
\end{equation} 
for every integer $k$. Correspondingly, 
\begin{equation*}
	0=\lambda_1\leq\lambda_2\leq\dots\leq\lambda_k \leq \cdots
\end{equation*}
represents the eigenvalues associated with the Laplacian $\Delta$, where $\lambda_k \to+\infty$ as $k \to \infty$.

Fixing a specific $t>0$, our attention is directed towards the diffusion space on $\mathcal{M}$, represented by
\begin{equation}\nonumber
	\mathcal{H}_t=e^{-\frac{t}{2}\Delta}L^2(\mathcal{M};\mu_g).
\end{equation}
This space is endowed with an inner product, defined as
\begin{equation*}
	\langle f,g\rangle_{\mathcal{H}_t}=\left\langle e^{\frac{t}{2}\Delta}f, e^{\frac{t}{2}\Delta}g\right\rangle_{L^2(\mathcal{M};\mu_g)},
\end{equation*}
thereby constituting a Hilbert space. Furthermore, $\mathcal{H}_t$ serves as an RKHS, in which the reproducing kernel, known as the heat kernel, is designated by $H_t(u,u')$. The definition of this kernel is given by
\begin{equation}\nonumber
	\int_{\mathcal{M}}H_t(u,u')f(u')d\mu_g(u')=\left(e^{-t\Delta}f\right)(u).
\end{equation}
In addition, we assume that the heat kernel is bounded, i.e., there exists a constant $\kappa \geq 1$, such that
\begin{equation}\label{Boundness of Heat Kernel}
	\sup\limits_{u\in\mathcal{M}}H_t(u,u)\leq\kappa^2.
\end{equation}
This assumption plays a crucial role in the subsequent analysis and applications of the heat kernel within our framework.
To illustrate, let us consider the specific case where $\mathcal{M}=S^m$ represents the m-dimensional unit sphere. In this context, the heat kernel on $S^m$ (for $m\geq3$) could be characterized as follows, detailed in \cite{cheng2018bessel} as
\begin{equation*}
	\begin{aligned}
		H^{S^m}_{t}(u,u')
		&=e^\frac{(m-1)^2t}{8}\left(\frac{r}{\sin r}\right)^\frac{m-1}{2}\frac{e^{-r^2/2t}}{(2\pi t)^{m/2}}\\
		&\quad\cdot\mathbb{E}_r \exp\left(-\frac{(m-1)(m-3)}{8}\int_0^t\left(\frac{1}{\sin^2 R_s}-\frac{1}{R^2_s}\right)ds\right).
	\end{aligned}
\end{equation*}
Here, $r=d_{S^m}(u,u')$, $R_s$ is indicative of an m-dimensional Bessel process, and $\mathbb{E}_r$ symbolizes the expectation conditioned on $R_t=r$. It is important to note that the original theorem presented in \cite{cheng2018bessel} pertains to the hyperbolic space $H^m$, yet it can be suitably adapted for application to the sphere $S^m$. We have chosen to omit these details, as they fall outside the primary scope of our discussion. Given that the exponent in the integrands within the expectation is consistently negative, we are justified in setting $\kappa^2=e^\frac{(m-1)^2t}{8}/(2\pi t)^{m/2}$. In fact, for a generic compact manifold $\mathcal{M}$, its Ricci curvature and sectional curvature are intrinsically bounded. This inherent boundedness assures that the condition on the boundedness of heat kernel invariably satisfied, which follows from the heat kernel comparison theorems. We refer to \cite{hsu2002stochastic} for more details.

A pivotal attribute of the diffusion space $\mathcal{H}_t$ is its ability to be embedded, as outlined in \cite{de2021reproducing}, which is given by
\begin{equation}\label{Embedding Property of Diffusion Space}
	\mathcal{H}_{t'}\hookrightarrow\mathcal{H}_t\hookrightarrow W^s(\mathcal{M}).
\end{equation}
This embedding holds true for all values of $0<t<t'$ and $s>0$, where $W^s(\mathcal{M})$ signifies the $s$-order Sobolev space on $\mathcal{M}$ under the $L_{\mu_g}^2$-norm. For every $s>\frac{m}{2}$, the Sobolev space $W^s(\mathcal{M})$ is compactly embedded into $C(\mathcal{M})$, which has been thoroughly discussed in \cite{triebel1992functionspace}. Given that $\mathcal{M}$ is compact, it follows that any diffusion space can be effectively embedded into $L^\infty(\mathcal{M};\mu_g)$. This implies that for all values of $t>0$, there exists a constant $A$ (which depends on $t$) such that
\begin{equation*}
	\|i:\mathcal{H}_t\hookrightarrow L^\infty\|\leq A.
\end{equation*}
In this context, $\|\cdot\|$ represents the operator norm. This result plays a crucial role in our analysis of upper bounds. The practical implementation and significance of this embedding property have been historically traced to the work presented in \cite{fischer2020sobolev}.

Let the output space $Y$ be the real line $\mathbb{R}$. We postulate that the unknown probability distribution $P$ on $\mathcal{M}\times\mathbb{R}$ adheres to the condition 
\begin{equation*}
	|P|_2^2=\int_{\mathcal{M}\times\mathbb{R}}y^2dP(x,y)<\infty.
\end{equation*} 
The marginal distribution of $P$ on $\mathcal{M}$ is denoted by $\nu$, and the regular conditional probability of $P$ given $x\in\mathcal{M}$ is represented as $P(\cdot|x)$. Furthermore, we presume that the marginal distribution $\nu$ corresponds to the uniform distribution on $\mathcal{M}$. In other words, $\nu$ varies from $\mu_g$ merely by multiplying a constant factor $p=\frac{1}{\text{vol}\mathcal{M}}$. This congruence ensures that $L^2(\mathcal{M};\mu_g)$ aligns with $L^2(\mathcal{M};\nu)$, and hence, we will omit the symbol $\mathcal{M}$ in subsequent discussions for the sake of conciseness.
Additionally, through direct computation, the regression function is formulated as
\begin{equation}\nonumber
	f^*(x)=\int_{\mathbb{R}}y\;dP(y|x).
\end{equation}

We now turn our attention to the inclusion map $I_\nu$ from $\mathcal{H}_t$ to $L^2(\nu)$. The adjoint of this operator, denoted as $I_\nu^*$, functions as an integral operator mapping from $L^2(\nu)$ to $\mathcal{H}_t$, as defined by 
\begin{equation}\nonumber
	(I_\nu^*f)(x)=\int_{\mathcal{M}}H_t(x,z)f(z)d\nu(z).
\end{equation}
Both $I_\nu$ and $I_\nu^*$ qualify as Hilbert-Schmidt operators, which inherently implies their compactness. The Hilbert-Schmidt norm of these operators conforms to the relationship
\begin{equation*}
	\|I_\nu^*\|_{\mathscr{L}^2(L^2(\nu),\mathcal{H}_t)}=
	\|I_\nu\|_{\mathscr{L}^2(\mathcal{H}_t,L^2(\nu))}=
	\|H_t\|_{L^2(\nu)}=\left(\int_\mathcal{M}H_t(x,x)d\nu(x)\right)^\frac{1}{2}
	\leq\kappa.
\end{equation*}
In addition, we introduce two supplementary integral operators, represented by 
\begin{equation}\nonumber
	L_\nu=I_\nu I_\nu^*:L^2(\nu)\to L^2(\nu)
\end{equation}
and 
\begin{equation}\label{Definition of T}
	T_\nu=I_\nu^* I_\nu:\mathcal{H}_t\to \mathcal{H}_t.
\end{equation}
These operators, $L_\nu$ and $T_\nu$, are characterized as self-adjoint, positive-definite, and fall within the trace class. This classification suggests that they are also Hilbert-Schmidt and compact, with their trace norm adhering to
\begin{equation*}
	\|T_\nu\|_{\mathscr{L}^1(\mathcal{H}_t)}=
	\|L_\nu\|_{\mathscr{L}^1(L^2(\nu))}=
	\|I_\nu^*\|^2_{\mathscr{L}^2(L^2(\nu),\mathcal{H}_t)}=
	\|I_\nu\|^2_{\mathscr{L}^2(\mathcal{H}_t,L^2(\nu))}\leq\kappa^2.
\end{equation*}
It is critical to highlight that $L_\nu$ is the integral operator corresponding to the heat kernel $H_t$ under the uniform distribution $\nu$. Moreover, by definition, the integral operator associated with $H_t$ under the Radon measure $\mu_g$ is explicitly expressed as $e^{-t\Delta}$. Given that $\nu=p\mu_g$ holds, it is straightforward to verify that $\{pe^{-t\lambda_k},\; p^{-\frac{1}{2}}\varphi_k\}_{k\in\mathbb{N}}$ constitutes an eigensystem of $L_\nu$ in $L^2(\nu)$. Consequently, $\{p^{-\frac{1}{2}}\varphi_k\}_{k\in \mathbb{N}}$ serves as an orthonormal basis of $L^2(\nu)$, and we denote them as $f_k = p^{-\frac{1}{2}}\varphi_k $.

Given that both $L_\nu$ and $T_\nu$ are self-adjoint and compact operators, the spectral theorem enables us to deduce their respective spectral decompositions, which are expressed as 
\begin{equation}\nonumber
	L_\nu=\sum_{k=1}^\infty pe^{-t\lambda_k}\langle \cdot, f_k\rangle_{L^2(\nu)}\cdot f_k,
\end{equation} 
and
\begin{equation}\label{Spectral Decomposition of T}
	T_\nu=\sum_{k=1}^\infty pe^{-t\lambda_k}\langle \cdot, p^{\frac{1}{2}}e^{-\frac{t\lambda_k}{2}}f_k\rangle_{\mathcal{H}_t} \cdot p^{\frac{1}{2}}e^{-\frac{t\lambda_k}{2}}f_k.
\end{equation}
Moreover, we have the capability to articulate $I_\nu^*$ in the form of 
\begin{equation}\nonumber
	I_\nu^*=\sum_{k=1}^\infty p^{\frac{1}{2}}e^{-\frac{t\lambda_k}{2}}\langle \cdot, f_k\rangle_{L^2(\nu)} \cdot p^{\frac{1}{2}}e^{-\frac{t\lambda_k}{2}}f_k.
\end{equation}
Crucially, by applying Mercer's theorem, the heat kernel can be expanded, both uniformly and absolutely (as elaborated in \cite{cucker2007learning,steinwart2008support}), as 
\begin{equation}\label{Spectral Decomposition of Heat Kernel}
	H_t(x,z)=\sum_{k=1}^\infty pe^{-t\lambda_k}f_k(x)f_k(z).
\end{equation}
Furthermore, $\{p^{\frac{1}{2}}e^{-\frac{t\lambda_k}{2}}f_k\}_{k\in\mathbb{N}}$ establishes an orthonormal basis for $\mathcal{H}_t$, thereby providing a foundational structure for further analysis and application within this framework.

We will now turn our attention to the $\alpha$-power space denoted by $\mathcal{H}_t^{\alpha}=\text{Ran}(L_\nu^\frac{\alpha}{2})$, which is defined as
\begin{equation}\label{Definition of Alpha-Power Space}
	\mathcal{H}_t^{\alpha}=\left\{
	\sum_{k=1}^\infty a_k p^{\frac{\alpha}{2}}e^{-\frac{\alpha\lambda_k}{2}t}f_k:\{a_k\}_{k=1}^\infty\in l^2(\mathbb{N})
	\right\}.
\end{equation}
This space, $\mathcal{H}_t^{\alpha}$, will be equipped with an $\alpha$-power norm, specified as
\begin{equation}\label{Definition of Alpha-Power Norm}
		\left\|\sum_{k=1}^\infty a_k p^{\frac{\alpha}{2}}e^{-\frac{\alpha\lambda_k}{2}t}f_k\right\|_{\alpha}:=\left\|\sum_{k=1}^\infty a_k p^{\frac{\alpha}{2}}e^{-\frac{\alpha\lambda_k}{2}t}f_k\right\|_{\mathcal{H}_t^{\alpha}}
	=\|(a_k)\|_{l^2(\mathbb{N})}
	=\left(\sum_{k=1}^\infty a_k^2\right)^\frac{1}{2}.
\end{equation}
Consequently, this leads us to the following relationship
\begin{equation}\nonumber
	\|L_\nu^{\frac{\alpha}{2}}(f)\|_\alpha=\|f\|_{L^2(\nu)}.
\end{equation}
With the introduction of this $\alpha$-power norm, $\mathcal{H}_t^{\alpha}$ is constituted as a Hilbert space, complete with an orthonormal basis $\{p^{\frac{\alpha}{2}}e^{-\frac{\alpha\lambda_k}{2}t}\}_{k\in\mathbb{N}}$. For ease of reference in our subsequent discussions, we will employ the abbreviation $\|\cdot\|_{\alpha}$ to represent the $\alpha$-power norm. It is essential to acknowledge that both $\mathcal{H}_t^{\alpha}=\mathcal{H}_{\alpha t}$ and $\mathcal{H}_t^{0}=L^2(\nu)$ hold true.
Further, due to the embedding property of the diffusion space, we can establish that for all $0<\beta<\alpha$, the following is applicable:
\begin{equation*}
	\mathcal{H}_t^\alpha\hookrightarrow\mathcal{H}_{t}^\beta\hookrightarrow L^2(\nu).
\end{equation*}
This embedding property can also be deduced directly from the definition of $\mathcal{H}_t^{\alpha}$. In fact, the $\alpha$-power space $\mathcal{H}_t^{\alpha}$ can be effectively described using the real interpolation method, as detailed in \cite{steinwart2012mercer}, given by
\begin{equation*}
	\mathcal{H}_t^\alpha=[L^2(\nu), \mathcal{H}_t]_{\alpha,2}.
\end{equation*}

To provide an in-depth understanding of the spectral algorithm, it is imperative to first introduce the concept of the regularization family, often referred to as filter functions. A set of functions denoted by $\{g_\lambda:\mathbb{R^+\to\mathbb{R}^+}\}_{\lambda>0}$ is recognized as a regularization family if it satisfies specific criteria outlined in
\begin{equation}\label{Property of Regularization Family}
	\begin{aligned}
		&\sup\limits_{0<t\leq\kappa^2}|tg_\lambda(t)|<1,\\
		&\sup\limits_{0<t\leq\kappa^2}|1-tg_\lambda(t)|<1,\\
		&\sup\limits_{0<t\leq\kappa^2}|g_\lambda(t)|<\lambda^{-1}.\\
	\end{aligned}
\end{equation}
For a given regularization family $g_\lambda$, we define its qualification $\xi$ as the supremum of the set
\begin{equation}\label{Qualification of Regularization Family}
	\xi = \sup \left\{\alpha>0:\sup\limits_{0<s\leq\kappa^2}|1-sg_\lambda(s)|s^\alpha<\lambda^\alpha\right\}.
\end{equation}
To exemplify these concepts, let us consider a few commonly utilized regularization families. The first example is $g_\lambda(t)=\frac{1}{\lambda+t}$, which possesses a qualification of $\xi=1$ and aligns with kernel ridge regression (also known as Tikhonov regularization). Another notable instance is $g_\lambda(t)=\frac{1}{t}\mathbbm{1}_{\{t\geq\lambda\}}$, characterized by a qualification of $\xi=\infty$, and is associated with kernel principal component regularization (or spectral cut-off). These examples offer a more comprehensive insight into the functionality of different regularization families within the framework of spectral algorithms. For additional examples and a deeper exploration of this topic, please refer to \cite{gerfo2008spectral}.

In our final formulation of the spectral algorithm, we start with a dataset $ D=\{(x_i,y_i)\}_{i=1}^n $ that is independently and identically distributed from the distribution $P$. Initially, we define the sampling operator $H_{t,x}:y\mapsto yH_t(x,\cdot)$ that maps from $\mathbb{R}$ to $\mathcal{H}_t$, as well as its adjoint operator $H_{t,x}^*:f\mapsto f(x)$ that maps from $\mathcal{H}_t$ to $\mathbb{R}$. Further, we establish the sample covariance operator $T_\delta:\mathcal{H}_t\to\mathcal{H}_t$, which is described by
\begin{equation}\nonumber
	T_\delta =\frac{1}{n}\sum_{i=1}^nH_{t,x_i}H_{t,x_i}^*.
\end{equation}
It is important to note that $T_\delta$ is in alignment with the integral operator defined in \eqref{Definition of T}, but with respect to the empirical marginal distribution $\delta=\frac{1}{n}\sum_{i=1}^n\delta_{x_i}$. In addition, we introduce the sample basis function as
\begin{equation}\nonumber
	g_D=\frac{1}{n}\sum_{i=1}^ny_iH_t(x_i,\cdot)\in\mathcal{H}_t.
\end{equation}
Utilizing the sample function and operator outlined above, we can now define the spectral algorithm estimator, which is given by
\begin{equation}\label{Definition of Spectal Alogrithm Estimator}
	f_{D,\lambda} = g_\lambda(T_\delta)g_D.
\end{equation}
In this setup, $g_\lambda(T_\delta)$ suggests that $g_\lambda$ interacts with $T_\delta$ through functional calculus. Significantly, \eqref{Definition of Spectal Alogrithm Estimator} offers a finite, computable representation, the calculation of which entails an eigen-decomposition of the kernel matrix. This provides a practical approach to implementing the spectral algorithm in computational settings.

It is noteworthy that deriving an explicit representation of the heat kernel $H_t$ for a general Riemannian manifold $\mathcal{M}$ presents a significant challenge. This complexity poses a potential obstacle in the development of our spectral algorithm, which depends on the kernel function $H_t(x,\cdot)$. To circumvent this limitation, we propose employing graph Laplacian techniques as a means to approximate the heat kernel. We plan to elaborate on this approach and its implementation in detail in our forthcoming research.

\section{Main Results}\label{section: main results}

\begin{assumption}\label{Assumption: Moment Condition}
	There exists positive constants $\sigma, L$ such that
	\begin{equation}\label{Moment Condition}
		\int_\mathbb{R} |y-f^*(x)|^m dP(y|x)\leq \frac{1}{2}m!\sigma^2L^{m-2}
	\end{equation}
	for $\nu$-almost every $x\in\mathcal{M}$ and all $m\geq2$. 
\end{assumption}

The moment condition we adopt here is a widely accepted assumption, principally designed to manage the discrepancy between observed data and the true underlying value. This condition guarantees that the tail probability of the noise diminishes at a sufficiently rapid pace, a practice that has been embraced in several scholarly works, as documented in \cite{wainwright2019high}. Crucially, this moment condition is typically met in scenarios involving Gaussian noise characterized by bounded variance, as well as in distributions that are concentrated within certain bounded intervals on the real line. This makes it a mild and reliable assumption in a variety of statistical contexts.

\begin{assumption}\label{Assumption: Source Condition}
	For $\beta>0$, there exists a constant $R>0$ such that $f^*\in\mathcal{H}_t^{\beta}$ and 
	\begin{equation}\label{Source Condition}
		\|f^*\|_\beta\leq R.
	\end{equation}
\end{assumption}

The source condition serves as a conventional metric for assessing the smoothness of the regression function, a methodology that has gained widespread adoption, as evidenced by the literature \cite{bauer2007regularization,caponnetto2007optimal}. It is important to recognize that lower values of $\beta$ correspond to a decrease in the smoothness of the regression function. This reduction in smoothness presents greater challenges for algorithms in terms of achieving accurate estimations. In instances where $\beta<1$, the situation is classified as a hard learning scenario. This particular scenario has recently garnered considerable interest among researchers, as reflected in studies and discussions found in Section \ref{section: introduction}.

We are now prepared to unveil our principal findings. In light of our emphasis on the hard learning scenario, which is indicated by $0<\beta<1$, our theorem will be specifically tailored to the range encapsulated by $0<\beta\leq1$. Nonetheless, it is important to note that similar results could be obtained for $\beta>1$ with minimal adjustments. This approach ensures that our conclusions are both relevant and adaptable to varying scenarios within the scope of our study.

\begin{theorem}\label{Theorem: Upper Bound}
	Suppose that the regularization family $\{g_\lambda\}_{\lambda>0}$ has qualification $\xi\geq\frac{1}{2}$, the moment condition \eqref{Moment Condition} in Assumption \ref{Assumption: Moment Condition} holds for some positive constants $\sigma, L$, and the source condition \eqref{Source Condition} in Assumption \ref{Assumption: Source Condition} holds for some $R>0$ and $0<\beta\leq1$. For all $0\leq\gamma<\beta$, suitably choose a regularization parameter sequence $\{\lambda_n\}$ such that
	\begin{equation*}
		\lambda_n\sim\left(\frac{(\log n)^\frac{m}{2}}{n}\right)^\frac{1}{\beta}.
	\end{equation*}
	Then, for all $t>0$, $\tau\geq1$, the spectral algorithm estimators $f_{D,\lambda_n}$ with respect to RKHS $\mathcal{H}_t$ satisfies
	\begin{equation}\nonumber
		\|f_{D,\lambda_n}-f^*\|_\gamma^2\leq C \tau^2\left(\frac{(\log n)^\frac{m}{2}}{n}\right)^\frac{\beta-\gamma}{\beta}
	\end{equation}
	for sufficiently large $n\geq1$ with probability at least $1-4e^{-\tau}$, where $C>0$ is a constant independent of $n, \tau$.
\end{theorem}

If the regression function is such that $f^*\in\mathcal{H}_s$ holds for some $s>0$, we can appropriately select and fix an $t_o>s$ and employ the spectral algorithm estimator \eqref{Definition of Spectal Alogrithm Estimator} to estimate the regression function in the RKHS $\mathcal{H}_{t_0}$. Drawing from Theorem \ref{Theorem: Upper Bound}, for any $t<s$, we are able to establish the following convergence rates in $\mathcal{H}_t$:
\begin{equation}\label{H_t Learning Rate}
	\|f_{D,\lambda_n}-f^*{}\|_{\mathcal{H}_t}^2\lesssim\left(\frac{(\log n)^\frac{m}{2}}{n}\right)^\frac{s-t}{s}.
\end{equation}  
It is significant to note that this $\mathcal{H}_t$-convergence rate \eqref{H_t Learning Rate} is not contingent upon the specific RKHS $\mathcal{H}_{t_0}$ previously selected; the only requirement being placed on the parameter $t_o$ is $t_o>s$.
Additionally, given that
\begin{equation*}
	\mathcal{H}_t\hookrightarrow W^s(\mathcal{M})\hookrightarrow C^k(\mathcal{M})
\end{equation*}
holds true for all $t>0$ and $s>k+\frac{m}{2}$, we also obtain $C^k$-convergence rates for any $k$-th order derivatives:
\begin{equation}\nonumber
	\|f_{D,\lambda_n}-f^*\|_{C^k}^2\lesssim\left(\frac{(\log n)^\frac{m}{2}}{n}\right)^{1-\varepsilon}.
\end{equation} 
In this context, the parameter $0<\varepsilon<1$ is derived from the stipulations of $\beta>\gamma$ in Theorem \ref{Theorem: Upper Bound}. This finding underscores the versatility and robustness of the proposed spectral algorithm in various analytical contexts.

We are now prepared to introduce the minimax lower bound.

\begin{theorem}\label{Theorem: Lower Bound}
	Suppose that $\nu$ is the uniform distribution on $\mathcal{M}$. Then, for all positive $\sigma, L, R$, and $0\leq\gamma<\beta\leq1$, there exists a constant $C>0$ such that, for all estimators $\hat{f}:=\{(x_i,y_i)\}_{i=1}^n\mapsto \hat{f}$, all $\tau\in(0,1)$, and all sufficiently large $n\geq 1$, there exists a probability distribution $P$ on $\mathcal{M}\times\mathbb{R}$ with marginal distribution $\nu$ on $\mathcal{M}$ satisfying the moment condition \eqref{Moment Condition} in Assumption \ref{Assumption: Moment Condition} with respect to $\sigma, L$ and the source condition \eqref{Source Condition} in Assumption \ref{Assumption: Source Condition} with respect to $R, \beta$ such that, with $P^{n}$-probability at least $1-\tau^2$, it holds
	\begin{equation}\nonumber
		\|\hat{f}-f^*\|^2_{\gamma}\geq C \tau^2 \frac{(\log n)^\frac{m}{2}}{n^s}
	\end{equation}
	with $s=1$ for $\gamma=0$ and 
	$s>\frac{(\beta-\gamma)2^{\frac{2}{m}}}{\gamma+(\beta-\gamma)2^{\frac{2}{m}}}$ for $\gamma>0$.
\end{theorem}

Integrating the insights from Theorem \ref{Theorem: Upper Bound} and Theorem \ref{Theorem: Lower Bound}, we discern that the spectral algorithm estimator \eqref{Definition of Spectal Alogrithm Estimator} aligns with the minimax lower bound, particularly in the scenario defined by $\gamma=0$. This scenario, described as the $L^2$-convergence, is articulated by
\begin{equation}\nonumber
	\|f_{D,\lambda_n}-f^*\|_{L^2(\nu)}^2\sim
	\frac{(\log n)^\frac{m}{2}}{n}.
\end{equation}
Such an alignment not only corroborates but also extends previous findings related to exponentially decayed or Gaussian decayed eigenvalues. These earlier results, with decay rates of $O\left(\frac{\log n}{n}\right)$ corresponding to $m=2$ and rates of $O\left(\frac{\sqrt{\log n}}{n}\right)$ corresponding to $m=1$ (as explored in \cite{dicker2017kernel}), are now applicable to manifolds of higher dimensions $m\geq2$.

However, in the more general setting of $\gamma>0$, challenges emerge due to the exponential function's rapid growth property, particularly highlighted in \eqref{Upper Bound: Gamma Norm of f_i} of our proof. Specifically, the issue arises because $e^{2x}$ cannot be effectively controlled by a constant multiple of $e^x$. This contrasts with approaches in \cite{blanchard2018optimal,fischer2020sobolev}, where eigenvalues demonstrate polynomial decay. Consequently, our methods fall short in providing a precise bound for the constructed Gaussian type probability distribution. This shortfall is evident in the proof of Theorem \ref{Theorem: Lower Bound} when aimed at establishing optimality for the general case $\gamma>0$. 
To effectively manage the sum of eigenvalues of the integral operator in \eqref{Upper Bound: Gamma Norm of f_i} and thereby derive a more stringent lower bound, a more advanced analytical approach is required. It is our hope that future research will be able to bridge this gap and offer a comprehensive solution to this complex problem.

\section{Upper Bounds Analysis}\label{Section: Upper Bounds Analysis}

Let us initially examine the continuous version of the estimator represented by $f_{D,\lambda}$, which we denote as $f_{P,\lambda}$. To define this, we consider
\begin{equation}\label{Continuous Version of Spectral Alogrithm Estimator}
	f_{P,\lambda} = g_\lambda(T_\nu)g_P.
\end{equation}
Here, $g_P$ signifies the continuous variant of the sample function $g_D$. The expression for this continuous version is given by
\begin{equation*}
	g_P=I_\nu^*f^*=\int_{\mathcal{M}\times\mathbb{R}}yH_t(x,\cdot)dP(x,y).
\end{equation*}
The $\gamma$-norm error can be effectively decomposed into two distinct components:
\begin{equation}\label{Decomposition of Gamma-Norm Error}
	\|f_{D,\lambda}-f^*\|_\gamma\leq\|f_{D,\lambda}-f_{P,\lambda}\|_\gamma+\|f_{P,\lambda}-f^*\|_\gamma.
\end{equation}
In this decomposition, the first term encapsulates the estimation error, while the second term pertains to the approximation error. We intend to undertake a detailed estimation of these two terms separately, with the estimation error being addressed in section \ref{Subsection: Approximation Error} and the approximation error in section \ref{Subsection: Estimation Error}. This structured approach allows for a comprehensive analysis of the error components inherent in the estimator.

\subsection{Bounding approximation error}\label{Subsection: Approximation Error}

Recalling the spectral decomposition \eqref{Spectral Decomposition of T} of $T_\nu$, we have
\begin{equation*}
	T_\nu=\sum_{k=1}^\infty pe^{-t\lambda_k}\langle \cdot, p^{\frac{1}{2}}e^{-\frac{t\lambda_k}{2}}f_k\rangle_{\mathcal{H}_t} \cdot p^{\frac{1}{2}}e^{-\frac{t\lambda_k}{2}}f_k.
\end{equation*}
Through spectral calculus, this leads us to
\begin{equation*}
	g_\lambda(T_\nu)=\sum_{k=1}^\infty g_\lambda(pe^{-t\lambda_k})\langle \cdot, p^{\frac{1}{2}}e^{-\frac{t\lambda_k}{2}}f_k\rangle_{\mathcal{H}_t} \cdot p^{\frac{1}{2}}e^{-\frac{t\lambda_k}{2}}f_k.
\end{equation*}
Further, if we expand $f^*$ utilizing the orthonormal basis $\{f_k\}_{k\in\mathbb{N}}$ within $L^2(\nu)$, it results in the expression
\begin{equation}\label{Expansion of Regression Function}
	f^*=\sum_{k=1}^\infty a_kf_k.
\end{equation}
This expansion allows us to deduce the subsequent elaboration of $f_{P,\lambda}$, given by
\begin{equation}\label{Expansion of Continuous Version Estimator}
	f_{P,\lambda}=\sum_{k=1}^\infty g_\lambda(pe^{-t\lambda_k})p^{\frac{1}{2}}e^{-\frac{t\lambda_k}{2}}a_k\cdot p^{\frac{1}{2}}e^{-\frac{t\lambda_k}{2}}f_k.
\end{equation}
Consequently, this brings us to the conclusion encapsulated in
\begin{equation}\label{Expansion of Approximation Error}
	f^*-f_{P,\lambda}=\sum_{k=1}^\infty\left(1-g_\lambda(pe^{-t\lambda_k})pe^{-{t\lambda_k}}\right)a_k\cdot f_k.
\end{equation}

\begin{lemma}\label{Lemma: Gamma-Norm Upper Bound of Approximation Error}
	Suppose that $f^*\in\mathcal{H}_t^{\beta}$ for some $\beta>0$, $\{g_\lambda\}_{\lambda>0}$ has qualification $\xi>0$. Then, for all $\gamma\geq0$, $0<\frac{\beta-\gamma}{2}\leq\xi$, and $\lambda>0$, it holds
	\begin{equation}\label{Gamma-Norm Upper Bound of Approximation Error}
		\|f_{P,\lambda}-f^*\|_\gamma^2\leq\|f^*\|_\beta^2\cdot\lambda^{\beta-\gamma}.
	\end{equation}
\end{lemma}

\begin{proof}
	Starting with the expansion \eqref{Expansion of Approximation Error} of the approximation error, we can write
	\begin{equation*}
		\begin{aligned}
			\|f_{P,\lambda}-f^*\|_\gamma^2
			&=\left\|\sum_{k=1}^\infty\left[1-g_\lambda(pe^{-t\lambda_k})pe^{-{t\lambda_k}}\right]p^{-\frac{\gamma}{2}}e^{\frac{\gamma t\lambda_k}{2}}a_k\cdot p^{\frac{\gamma}{2}}e^{-\frac{\gamma t\lambda_k}{2}}f_k\right\|_\gamma^2\\
			&=\sum_{k=1}^\infty\left(\left[1-g_\lambda(pe^{-t\lambda_k})pe^{-{t\lambda_k}}\right]p^{-\frac{\gamma}{2}}e^{\frac{\gamma t\lambda_k}{2}}a_k\right)^2\\
			&=\sum_{k=1}^\infty\left(\left[1-g_\lambda(pe^{-t\lambda_k})pe^{-{t\lambda_k}}\right]p^{\frac{\beta-\gamma}{2}}e^{-\frac{(\beta-\gamma) t\lambda_k}{2}}\right)^2p^{-\beta}e^{{\beta t\lambda_k}}a_k^2.\\
		\end{aligned}
	\end{equation*}
	Given that $\{g_\lambda\}$ possesses a qualification of $\xi\geq\frac{\beta-\gamma}{2}$, and referring to \eqref{Qualification of Regularization Family}, we arrive at
	\begin{equation*}
		\left|\left[1-g_\lambda(pe^{-t\lambda_k})pe^{-{t\lambda_k}}\right]p^{\frac{\beta-\gamma}{2}}e^{-\frac{(\beta-\gamma) t\lambda_k}{2}}\right|<\lambda^\frac{\beta-\gamma}{2}.
	\end{equation*}	
	Thus, we can deduce
	\begin{equation*}
		\|f_{P,\lambda}-f^*\|_\gamma^2\leq\sum_{k=1}^\infty\lambda^{\beta-\gamma}p^{-\beta}e^{{\beta t\lambda_k}}a_k^2=\lambda^{\beta-\gamma}\|f^*\|_\beta^2.
	\end{equation*}
	The last equation is derived from combining the expansion  \eqref{Expansion of Regression Function} with the definition of the $\beta$-norm, as outlined in \eqref{Definition of Alpha-Power Norm}. The proof is then finished.
\end{proof}

Utilizing the spectral decomposition \eqref{Spectral Decomposition of Heat Kernel} of the heat kernel $H_t$:
\begin{equation*}
	H_t(x,z)=\sum_{k=1}^\infty pe^{-t\lambda_k}f_k(x)f_k(z).
\end{equation*}
our preceding boundedness assumption, \eqref{Boundness of Heat Kernel}, can be reformulated as
\begin{equation*}
	\sum_{k=1}^\infty pe^{-t\lambda_k}f^2_k(x)\leq\kappa^2.
\end{equation*}
This inequality serves as the foundation for introducing the concept of $\alpha$-boundedness of the heat kernel $H_t$, a topic explored in greater depth in \cite{steinwart2012mercer}. Specifically, we define $H_t$ as being $\alpha$-bounded by a constant $A$ if the condition
\begin{equation}\label{Alpha Boundness}
	\sum_{k=1}^\infty p^{\alpha}e^{-\alpha t\lambda_k}f^2_k(x)\leq A^2
\end{equation}
is met for $\nu$-almost every $x\in\mathcal{M}$. Additionally, the smallest constant $A$ that satisfies \eqref{Alpha Boundness} is denoted as $\|H_t^\alpha\|_\infty$. A key aspect of $\alpha$-boundedness is its connection to the embedding from the corresponding RKHS $\mathcal{H}_t^\alpha$ into $L^\infty(\nu)$. This relationship is elaborated upon in the following lemma, which is extracted from \cite{fischer2020sobolev}. This lemma plays a pivotal role in understanding the underlying dynamics of the heat kernel within the framework of RKHS.

\begin{lemma}\label{Lemma: Alpha Boundness and Embedding into L-infty}
	For all $\alpha>0$, consider the inclusion map $i_\alpha: \mathcal{H}_t^\alpha\to L^\infty(\nu)$, then we have 
	\begin{equation*}
		\|i_\alpha\|=\|H_t^\alpha\|_\infty.
	\end{equation*}
\end{lemma}

Given the established fact that $\mathcal{H}_t^\alpha$ is continuously embedded into $L^\infty(\nu)$ for all $t>0$ and $\alpha>0$, as inferred from the embedding property of the diffusion space \eqref{Embedding Property of Diffusion Space}, we are at liberty to select and fix a value for $\alpha$ such that $0<\alpha<\beta$ holds. Subsequently, we denote the corresponding smallest $\alpha$-bound $\|H_t^\alpha\|_\infty$ as $A_\alpha$.
As a result of this setup, and by integrating the insights from Lemma \ref{Lemma: Gamma-Norm Upper Bound of Approximation Error} and Lemma \ref{Lemma: Alpha Boundness and Embedding into L-infty}, we can further constrain the approximation error in the $L^\infty$-norm, which is articulated as
\begin{equation}\label{L-infty Bound of Approximation Error}
	\|f_{P,\lambda}-f^*\|_{L^\infty(\nu)}^2\leq A_\alpha^2\|f^*\|_\beta^2\cdot\lambda^{\beta-\alpha}.
\end{equation}
This step is crucial in effectively managing the estimation error, a task that we will delve into in Subsection \ref{Subsection: Estimation Error}. The ability to bound the approximation error in this manner plays a pivotal role in enhancing the precision and reliability of our overall analysis in the subsequent section.

Finally, we turn our attention to the effective dimension, represented by
\begin{equation}\label{Effective Dimension}
	N_\nu(\lambda)= \mathop{\rm tr}\left((T_\nu+\lambda)^{-1}T_\nu\right)=\sum_{k=1}^\infty\frac{pe^{-t\lambda_k}}{\lambda+pe^{-t\lambda_k}}.
\end{equation}
The decay rate of this dimension as $\lambda\to0$ is a critical factor in our analysis of the estimation error. To establish an upper bound for the effective dimension, it is necessary to apply Weyl's law. This application is encapsulated in the following Lemma \ref{Lemma: Eigenvalue Growthness of Laplacian}, which provides a means to control the growth rate of the eigenvalues $\lambda_k$ of the Laplacian on the $m$-dimensional manifold $\mathcal{M}$. For a comprehensive understanding of Weyl's law, especially in the context of classical Riemannian manifolds, one can refer to  \cite{li1983schrodinger,chavel1984eigenvalues}. These references offer in-depth discussions and insights into the application and implications of Weyl's law, enriching the understanding of its role in controlling the eigenvalue growth and, consequently, in bounding the effective dimension in our analysis.

\begin{lemma}\label{Lemma: Eigenvalue Growthness of Laplacian}
	For the eigenvalue $\lambda_k$  of the Laplacian $\Delta$ on a compact, connected Riemann manifold $\mathcal{M}$ with dimension $m$, we have the following asymptotic estimation
	\begin{equation}\label{Eigenvalue Growthness of Laplacian}
		C_{low}k^\frac{2}{m}\leq\lambda_k\leq C_{up} k^\frac{2}{m}
	\end{equation}
	for all $k\in \mathbb{N}$ with $C_{low}$ and $C_{up}$ two absolute constants depending only on $\mathcal{M}$.
\end{lemma}

Utilizing Weyl's law, as previously discussed, allows us to derive the necessary upper bound for the effective dimension $N_\nu(\lambda)$ that is integral to our analysis. This upper bound is articulated in the following Lemma \ref{Lemma: Upper Bound of Effective Dimension}.

\begin{lemma}\label{Lemma: Upper Bound of Effective Dimension}
	Suppose that $m\geq2$. Then, there exists a constant $D$ such that, for all $0<\lambda<\frac{1}{2}$, it holds
	\begin{equation}\nonumber
		N_\nu(\lambda)\leq D(\log\lambda^{-1})^\frac{m}{2}.
	\end{equation}
	In this context, the constant $D$ only rely on $\mathcal{M}$ and $t$.
\end{lemma}

\begin{proof}
	Given that the function represented by $t\mapsto t/(t+\lambda)$ is increasing in $t$, and by combining \eqref{Effective Dimension} with \eqref{Eigenvalue Growthness of Laplacian}, we arrive at
	\begin{equation*}
		N_\nu(\lambda)=\sum_{k=1}^\infty\frac{pe^{-t\lambda_k}}{\lambda+pe^{-t\lambda_k}}
		\leq\sum_{k=1}^\infty\frac{pe^{-tC_{low}k^\frac{2}{m}}}{\lambda+pe^{-tC_{low}k^\frac{2}{m}}}
		=\sum_{k=1}^\infty\frac{1}{1+\lambda p^{-1}e^{tC_{low}k^\frac{2}{m}}}.
	\end{equation*}
	Considering that the function $x\mapsto 1/(1+\lambda p^{-1}e^{tC_{low}x^\frac{2}{m}})$ is decreasing in $ x $ and remains positive, we can set an upper bound for the summation using an integral, expressed as
	\begin{equation*}
		N_\nu(\lambda)\leq\int_{0}^\infty\frac{1}{1+\lambda p^{-1}e^{tC_{low}x^\frac{2}{m}}}dx
		=\frac{m}{2}(C_{low}t)^{-\frac{m}{2}}\int_{0}^\infty\frac{u^{\frac{m}{2}-1}}{1+\lambda p^{-1}e^u}du.
	\end{equation*}
	In above equation, the remaining integral on the right-hand side represents the complete Fermi-Dirac integral. To evaluate this integral, we utilize the polylogarithm function
	\begin{equation*}
		{\rm Li}_s(z)=\sum_{k=1}^\infty\frac{z^k}{k^s}.
	\end{equation*}
	Drawing from \cite{wood1992computation}, we can transform the complete Fermi-Dirac integral into an expression involving the polylogarithm function and Gamma function, resulting in
	\begin{equation*}
		\int_{0}^\infty\frac{u^{\frac{m}{2}-1}}{1+p^{-1}\lambda e^u}du
		=-{\rm Li}_{\frac{m}{2}}(-p\lambda^{-1})\cdot\Gamma(\frac{m}{2}).
	\end{equation*}
	Given that the polylogarithm exhibits a specific limit behavior, as also referenced in \cite{wood1992computation}, we deduce
	\begin{equation*}
		{\rm Li}_{\frac{m}{2}}(-p\lambda^{-1})\sim-\frac{1}{\Gamma(\frac{m}{2}+1)}\left(\log (p\lambda^{-1})\right)^\frac{m}{2}
	\end{equation*}
	as $\lambda\to0$. 
	Consequently, since $ \Gamma(s+1)=s\Gamma(s) $, there exists a constant $ \lambda_0\ll1 $ such that, for all $ 0<\lambda\leq\lambda_0$ it holds
	\begin{equation*}
		\begin{aligned}
			N_\nu(\lambda)
			&\leq1.1(C_{low}t)^{-\frac{m}{2}}\left(\log (p\lambda^{-1})\right)^\frac{m}{2}\\
			&\leq1.1(C_{low}t)^{-\frac{m}{2}}\cdot2^{\frac{m}{2}-1}\left((\log\lambda^{-1})^\frac{m}{2}+(\log p)^\frac{m}{2}\right)\\
			&\leq2^{\frac{m}{2}+1}(C_{low}t)^{-\frac{m}{2}}\cdot(\log\lambda^{-1})^\frac{m}{2}.
		\end{aligned}
	\end{equation*}
	As for $ \lambda_0<\lambda<1/2 $, due to the decreasing property of $ N_\nu(\lambda) $, we can find a constant $ D_0 $ (which only depends on $ \mathcal{M},t $) such that 
	\begin{equation*}
		N_\nu(\lambda)\leq N_\nu(\lambda_0)\leq D_0 (\log2)^\frac{m}{2}\leq D_0 (\log\lambda^{-1})^\frac{m}{2}.
	\end{equation*}
	Finally, taking $ D = \max\{2^{\frac{m}{2}+1}(C_{low}t)^{-\frac{m}{2}},D_0\} $ yields the desired result and completes the proof.
\end{proof}

\subsection{Bounding estimation error}\label{Subsection: Estimation Error}

Firstly, we will demonstrate that the $\gamma$-norm can be appropriately transformed into the $\mathcal{H}_t$-norm. This transformation is crucial as it allows for a more versatile and applicable analysis by bridging different norms.

\begin{lemma}\label{Lemma: Relation Between Gamma-Norm and H_t-Norm}
	For all $0\leq\gamma\leq1$, $f\in\mathcal{H}_t$ we have
	\begin{equation}\label{Relation Between Gamma-Norm and H_t-Norm}
		\|f\|_\gamma=\left\|T_\nu^\frac{1-\gamma}{2}f\right\|_{\mathcal{H}_t}.
	\end{equation}
\end{lemma}

\begin{proof}
	Let's denote
	\begin{equation*}
		f=\sum_{k=1}^\infty b_k p^{\frac{1}{2}}e^{-\frac{t\lambda_k}{2}}f_k
	\end{equation*}
	with $b_k=\langle f,p^{\frac{1}{2}}e^{-\frac{t\lambda_k}{2}}f_k\rangle_{\mathcal{H}_t}$. Given that $\{p^{\frac{\gamma}{2}}e^{-\frac{\gamma t\lambda_k}{2}}f_k\}_{k\in\mathbb{N}}$ constitutes an orthonormal basis of $\mathcal{H}_t^\gamma$, we can apply the Parseval identity, leading us to
	\begin{equation*}
		\|f\|_\gamma^2 = \left\|\sum_{k=1}^\infty b_k p^{-\frac{\gamma-1}{2}}e^{(\gamma-1)\frac{t\lambda_k}{2}}\cdot p^{\frac{\gamma}{2}}e^{-\gamma\frac{t\lambda_k}{2}}f_k\right\|_\gamma^2=\sum_{k=1}^\infty b_k^2p^{1-\gamma}e^{(\gamma-1)t\lambda_k}.
	\end{equation*}
	Similarly, when we incorporate \eqref{Spectral Decomposition of T} into our calculations, it results in
	\begin{equation*}
		\left\|T_\nu^\frac{1-\gamma}{2}f\right\|^2_{\mathcal{H}_t}=\left\|\sum_{k=1}^\infty p^{\frac{1-\gamma}{2}}e^{-t\lambda_k\frac{1-\gamma}{2}}b_k\cdot p^{\frac{1}{2}}e^{-\frac{t\lambda_k}{2}}f_k\right\|_{\mathcal{H}_t}^2
		=\sum_{k=1}^\infty b_k^2p^{1-\gamma}e^{(\gamma-1)t\lambda_k}
	\end{equation*}
	as desired. Then we complete the proof. 
\end{proof}

We are now ready to tackle the estimation error term. Starting with \eqref{Relation Between Gamma-Norm and H_t-Norm}, we decompose the estimation error into three distinct terms, as outlined below:
\begin{subequations}
	\begin{align}
		\|f_{D,\lambda}-f_{P,\lambda}\|_\gamma^2
		&=\left\|T_\nu^\frac{1-\gamma}{2}(f_{D,\lambda}-f_{P,\lambda})\right\|_{\mathcal{H}_t}^2\nonumber\\
		&=\left\|T_\nu^\frac{1-\gamma}{2}(T_\nu+\lambda)^{-\frac{1}{2}}(T_\nu+\lambda)^{\frac{1}{2}}(T_\delta+\lambda)^{-\frac{1}{2}}(T_\delta+\lambda)^{\frac{1}{2}}(f_{D,\lambda}-f_{P,\lambda})\right\|_{\mathcal{H}_t}^2\nonumber\\
		&\leq\left\|T_\nu^\frac{1-\gamma}{2}(T_\nu+\lambda)^{-\frac{1}{2}}\right\|_{\mathscr{B}(\mathcal{H}_t)}^2\cdot\left\|(T_\nu+\lambda)^{\frac{1}{2}}(T_\delta+\lambda)^{-\frac{1}{2}}\right\|_{\mathscr{B}(\mathcal{H}_t)}^2\label{Decomposition of Estimation Error: First Step; Term a}\\
		&\quad\cdot\left\|(T_\delta+\lambda)^{\frac{1}{2}}(f_{D,\lambda}-f_{P,\lambda})\right\|_{\mathcal{H}_t}^2\label{Decomposition of Estimation Error: First Step; Term b}.
	\end{align}
\end{subequations}
We provide an explanation for the symbols adopted here. Let $(\mathcal{H}, \langle\cdot,\cdot\rangle_{\mathcal{H}})$ and $(\mathcal{H}', \langle\cdot,\cdot\rangle_{\mathcal{H}'})$ be two Hilbert spaces. Hereinafter, the collection of bounded linear operators from $\mathcal{H}$ to $\mathcal{H}'$ forms a Banach space, symbolized as $\mathscr{B}(\mathcal{H}, \mathcal{H}')$,  when considered under operator norm $\|A\|_{\mathscr{B}(\mathcal{H}, \mathcal{H}')} = \sup_{\|f\|_{\mathcal{H}}=1} \|Af\|_{\mathcal{H}'}$. When $\mathcal{H} = \mathcal{H}'$, the space is then denoted by $\mathscr{B}(\mathcal{H})$ with the corresponding norm given by $\|A\|_{\mathscr{B}(\mathcal{H})}$.
Our approach will involve establishing upper bounds for these three terms in \eqref{Decomposition of Estimation Error: First Step; Term a} and \eqref{Decomposition of Estimation Error: First Step; Term b}, which will be addressed separately.

Let us begin with the estimation of the first term in \eqref{Decomposition of Estimation Error: First Step; Term a}. For this term, we utilize the spectral decomposition outlined in \eqref{Spectral Decomposition of T}, combined with direct calculations on the function $t\mapsto t^{1-\gamma}/(\lambda+t)$. This process leads us to the conclusion that
\begin{equation}\label{Upper Bound: First Term in Term a of First Step Decomposition of Estimation Error}
	\left\|T_\nu^\frac{1-\gamma}{2}(T_\nu+\lambda)^{-\frac{1}{2}}\right\|_{\mathscr{B}(\mathcal{H}_t)}^2=\sup\limits_{k\geq1}\frac{p^{1-\gamma}e^{-t\lambda_k(1-\gamma)}}{pe^{-t\lambda_k}+\lambda}
	\leq\lambda^{-\gamma},
\end{equation}
which is valid for all $0\leq\gamma\leq1$. Moving on to the second term in \eqref{Decomposition of Estimation Error: First Step; Term a}, we can assert
\begin{equation}\label{First Step Transformation of Second Term in Term a of First Step Decomposition of Estimation Error}
	\left\|(T_\nu+\lambda)^{\frac{1}{2}}(T_\delta+\lambda)^{-\frac{1}{2}}\right\|_{\mathscr{B}(\mathcal{H}_t)}^2
	=\left\|(T_\nu+\lambda)^{\frac{1}{2}}(T_\delta+\lambda)^{-1}(T_\nu+\lambda)^{\frac{1}{2}}\right\|_{\mathscr{B}(\mathcal{H}_t)}.
\end{equation}
This conclusion is derived from the fact that both $T_\nu$ and $T_\delta$ are self-adjoint operators. Additionally, it is important to note the transformation represented by 

\begin{equation*}
	\begin{aligned}
		(T_\nu+\lambda)^{\frac{1}{2}}(T_\delta+\lambda)^{-1}(T_\nu+\lambda)^{\frac{1}{2}}
		&=\left((T_\nu+\lambda)^{-\frac{1}{2}}(T_\delta+\lambda)(T_\nu+\lambda)^{-\frac{1}{2}}\right)^{-1}\\
		&=\left((T_\nu+\lambda)^{-\frac{1}{2}}(T_\delta-T_\nu+T_\nu+\lambda)(T_\nu+\lambda)^{-\frac{1}{2}}\right)^{-1}\\
		&=\left((T_\nu+\lambda)^{-\frac{1}{2}}(T_\delta-T_\nu)(T_\nu+\lambda)^{-\frac{1}{2}}+I\right)^{-1}.
	\end{aligned}
\end{equation*}
Combining with the Taylor series of function $ x\mapsto\frac{1}{1-x} $ at origin, we derive that
\begin{equation}\label{Second Step Transformation of Second Term in Term a of First Step Decomposition of Estimation Error}
  \begin{split}
	&\quad \left\|(T_\nu+\lambda)^{\frac{1}{2}}(T_\delta+\lambda)^{-1}(T_\nu+\lambda)^{\frac{1}{2}}\right\|_{\mathscr{B}(\mathcal{H}_t)}\\
	&=\left\|\left(I-(T_\nu+\lambda)^{-\frac{1}{2}}(T_\nu-T_\delta)(T_\nu+\lambda)^{-\frac{1}{2}}\right)^{-1}\right\|_{\mathscr{B}(\mathcal{H}_t)}\\
	&=\sum_{k=0}^\infty\left\|(T_\nu+\lambda)^{-\frac{1}{2}}(T_\nu-T_\delta)(T_\nu+\lambda)^{-\frac{1}{2}}\right\|_{\mathscr{B}(\mathcal{H}_t)}^k.
\end{split}
\end{equation}
Therefore, our analysis for the second term reduces to the estimation of the operator norm of the difference between $T_\nu$ and its empirical counterpart $T_\delta$. We employ Lemma \ref{Lemma: Empirical Error Lemma; First One} to accomplish it.
This lemma is grounded in Bernstein's type concentration inequality, which is applicable to random variables valued in Hilbert spaces. For a detailed discussion on this Bernstein's type inequality, one may refer to \cite{smale2007learning}.

\begin{lemma}\label{Lemma: Empirical Error Lemma; First One}
	Denote
	\begin{equation*}
		p_\lambda = \log\left(2eN_\nu(\lambda)\frac{\lambda+\|T_\nu\|}{\|T_\nu\|}\right).
	\end{equation*}
	Then, for all $n\geq1$, $\lambda>0$, and $\tau\geq1$, the following operator norm bound holds with probability not less than $1-2e^{-\tau}$:
	\begin{equation*}
		\left\|(T_\nu+\lambda)^{-\frac{1}{2}}(T_\nu-T_\delta)(T_\nu+\lambda)^{-\frac{1}{2}}\right\|_{\mathscr{B}(\mathcal{H}_t)}\leq\frac{4A_\alpha^2\tau p_\lambda}{3n\lambda^\alpha}+\sqrt{\frac{2A_\alpha^2\tau p_\lambda}{n\lambda^\alpha}}.
	\end{equation*}
\end{lemma}

 The proof of Lemma \ref{Lemma: Empirical Error Lemma; First One} is detailed in \cite{fischer2020sobolev}. Building upon the previous lemma, we consolidate our findings into Lemma \ref{Lemma: Upper Bound of Second Term in Term a of First Step Decomposition of Estimation Error} that provides an upper bound for the second term in \eqref{Decomposition of Estimation Error: First Step; Term a}.

\begin{lemma}\label{Lemma: Upper Bound of Second Term in Term a of First Step Decomposition of Estimation Error}
	Suppose that $ n\geq8A_\alpha^2\tau p_\lambda\lambda^{-\alpha} $. Then, for all $\tau\geq1$ and $\lambda>0$, the following operator norm bound holds with probability not less than $1-2e^{-\tau}$:
	\begin{equation}\label{Upper Bound: Second Term in Term a of First Step Decomposition of Estimation Error}
		\left\|(T_\nu+\lambda)^{\frac{1}{2}}(T_\delta+\lambda)^{-\frac{1}{2}}\right\|_{\mathscr{B}(\mathcal{H}_t)}^2\leq3.
	\end{equation}
\end{lemma}

\begin{proof}
	Given that $ n\geq8A_\alpha^2\tau p_\lambda\lambda^{-\alpha} $, from Lemma \ref{Lemma: Empirical Error Lemma; First One} we have 
	\begin{equation*}
		\left\|(T_\nu+\lambda)^{-\frac{1}{2}}(T_\nu-T_\delta)(T_\nu+\lambda)^{-\frac{1}{2}}\right\|_{\mathscr{B}(\mathcal{H}_t)}\leq\frac{4A_\alpha^2\tau p_\lambda}{3n\lambda^\alpha}+\sqrt{\frac{2A_\alpha^2\tau p_\lambda}{n\lambda^\alpha}}\leq\frac{2}{3}
	\end{equation*}		
	Incorporating with our preceding calculations \eqref{First Step Transformation of Second Term in Term a of First Step Decomposition of Estimation Error} and \eqref{Second Step Transformation of Second Term in Term a of First Step Decomposition of Estimation Error}, it is straightforward to observe that, under our specified assumptions, the following holds true
	\begin{equation*}
		\left\|(T_\nu+\lambda)^{\frac{1}{2}}(T_\delta+\lambda)^{-\frac{1}{2}}\right\|_{\mathscr{B}(\mathcal{H}_t)}^2=\sum_{k=0}^\infty\left\|(T_\nu+\lambda)^{-\frac{1}{2}}(T_\nu-T_\delta)(T_\nu+\lambda)^{-\frac{1}{2}}\right\|_{\mathscr{B}(\mathcal{H}_t)}^k\leq\sum_{k=0}^\infty \left(\frac{2}{3}\right)^k=3,
	\end{equation*}
	which is the desired result. The proof is then finished.
\end{proof}

Now, let us revisit the third term in \eqref{Decomposition of Estimation Error: First Step; Term b}. We recall the expressions \eqref{Definition of Spectal Alogrithm Estimator} and \eqref{Continuous Version of Spectral Alogrithm Estimator}, given by
\begin{equation*}
	f_{D,\lambda} = g_\lambda(T_\delta)g_D,
\end{equation*}
\begin{equation*}
	f_{P,\lambda} = g_\lambda(T_\nu)g_P
\end{equation*}
with $g_P=I_\nu^* f^*$. Let $h_\lambda(t)=1-tg_\lambda(t)$. Given that $h_\lambda(t)+tg_\lambda(t)=1$ holds, we arrive at the following result
\begin{equation*}
	\begin{aligned}
		f_{D,\lambda}-f_{P,\lambda}
		&=g_\lambda(T_\delta)g_D-\left(h_\lambda(T_\delta)+T_\delta g_\lambda(T_\delta)\right)f_{P,\lambda}\\
		&=g_\lambda(T_\delta)(g_D-T_\delta f_{P,\lambda}) - h_\lambda(T_\delta)f_{P,\lambda}.
	\end{aligned}
\end{equation*}
Consequently, this allows us to further decompose \eqref{Decomposition of Estimation Error: First Step; Term b} into two separate components:
\begin{align}
	&\left\|(T_\delta+\lambda)^{\frac{1}{2}}(f_{D,\lambda}-f_{P,\lambda})\right\|_{\mathcal{H}_t}^2\nonumber\\
	&\quad\leq2\left\|(T_\delta+\lambda)^{\frac{1}{2}}\left(g_\lambda(T_\delta)(g_D-T_\delta f_{P,\lambda})\right)\right\|_{\mathcal{H}_t}^2+2\left\|(T_\delta+\lambda)^{\frac{1}{2}}\left(h_\lambda(T_\delta)f_{P,\lambda}\right)\right\|_{\mathcal{H}_t}^2\label{Decomposition of Estimation Error: Second Step}.
\end{align}
Our immediate focus will be on addressing the second term in \eqref{Decomposition of Estimation Error: Second Step}. The result concerning the upper bound for this term is encapsulated in the subsequent Lemma \ref{Lemma: Upper Bound of First Term in Second Step Decomposition of Estimation Error}.

\begin{lemma}\label{Lemma: Upper Bound of First Term in Second Step Decomposition of Estimation Error}
	Suppose that $\{g_\lambda\}_{\lambda>0}$ has a qualification $\xi\geq\frac{1}{2}$. Then, for all $\lambda>0$ and $0<\beta\leq1$ it holds
	\begin{equation}\label{Upper Bound: First Term in Second Step Decomposition of Estimation Error}
		\left\|(T_\delta+\lambda)^{\frac{1}{2}}\left(h_\lambda(T_\delta)f_{P,\lambda}\right)\right\|_{\mathcal{H}_t}^2\leq4\lambda^\beta\|f^*\|^2_\beta.
	\end{equation}
\end{lemma}

\begin{proof}
	Given that $\{g_\lambda\}_{\lambda>0}$ possesses a qualification of $\xi\geq\frac{1}{2}$, we can apply \eqref{Qualification of Regularization Family} to obtain
	\begin{align}
		\left\|(T_\delta+\lambda)^{\frac{1}{2}}h_\lambda(T_\delta)\right\|_{\mathscr{B}(\mathcal{H}_t)}^2
		&\leq\left(\sup\limits_{0\leq t\leq\kappa^2}(t+\lambda)^\frac{1}{2}h_\lambda(t)\right)^2\nonumber\\
		&\leq\left(\sup\limits_{0\leq t\leq\kappa^2}(t^\frac{1}{2}+\lambda^\frac{1}{2})h_\lambda(t)\right)^2\nonumber\\
		&\leq\left(\lambda^\frac{1}{2}+\lambda^\frac{1}{2}\right)^2
		=4\lambda\label{Upper Bound: Operator Norm of Empirical Integral Operator}.
	\end{align}
	Here, the second inequality is derived from $(a+b)^p\leq a^p+b^p$, applicable whenever $0\leq p\leq1$ holds.
	Regarding the remaining $\mathcal{H}_t$-norm of $f_{P,\lambda}$, we refer to the expansion \eqref{Expansion of Continuous Version Estimator}, given by
	\begin{equation*}
		f_{P,\lambda}=\sum_{k=1}^\infty
		g_\lambda(pe^{-t\lambda_k})p^\frac{1}{2}e^{-\frac{t\lambda_k}{2}}a_k\cdot p^\frac{1}{2}e^{-\frac{t\lambda_k}{2}}f_k.
	\end{equation*}
	This leads us to
	\begin{align}
		\|f_{P,\lambda}\|_{\mathcal{H}_t}^2
		&=\left\|\sum_{k=1}^\infty g_\lambda(pe^{-t\lambda_k})p^\frac{1}{2}e^{-\frac{t\lambda_k}{2}}a_k\cdot p^\frac{1}{2}e^{-\frac{t\lambda_k}{2}}f_k\right\|_{\mathcal{H}_t}^2
		\nonumber\\
		&=\sum_{k=1}^\infty\left(g_\lambda(pe^{-t\lambda_k})p^\frac{1}{2}e^{-\frac{t\lambda_k}{2}}\right)^2a_k^2\nonumber\\
		&=\sum_{k=1}^\infty g_\lambda(pe^{-t\lambda_k})^{1-\beta}\cdot g_\lambda(pe^{-t\lambda_k})^{\beta+1}p^{\beta+1}e^{-(\beta+1)t\lambda_k}\cdot p^{-\beta}e^{\beta t\lambda_k}a_k^2\nonumber\\
		&\leq\sum_{k=1}^\infty \lambda^{\beta-1}\cdot p^{-\beta}e^{\beta t\lambda_k}a_k^2
		\nonumber\\
		&=\lambda^{\beta-1}\|f^*\|_\beta^2\label{Upper Bound: H_t Norm of Continuous Version of Estimator}.
	\end{align}
	In this context, the inequality is a result of the properties \eqref{Property of Regularization Family} associated with the regularization family.
	By synthesizing \eqref{Upper Bound: Operator Norm of Empirical Integral Operator} with \eqref{Upper Bound: H_t Norm of Continuous Version of Estimator}, we can achieve the desired result and complete the proof.
\end{proof}

Now, we shift our focus to the first term in \eqref{Decomposition of Estimation Error: Second Step}. To analyze this term, we will use the following decomposition:
\begin{subequations}
	\begin{align}
		&\left\|(T_\delta+\lambda)^{\frac{1}{2}}\left(g_\lambda(T_\delta)(g_D-T_\delta f_{P,\lambda})\right)\right\|_{\mathcal{H}_t}^2\nonumber\\
		&=\left\|(T_\delta+\lambda)^{\frac{1}{2}}g_\lambda(T_\delta)(T_\delta+\lambda)^{\frac{1}{2}}\cdot(T_\delta+\lambda)^{-\frac{1}{2}}(T_\nu+\lambda)^{\frac{1}{2}}\nonumber\cdot(T_\nu+\lambda)^{-\frac{1}{2}}\left(g_D-T_\delta f_{P,\lambda}\right)\right\|_{\mathcal{H}_t}^2\nonumber\\
		&\leq\left\|(T_\delta+\lambda)^{\frac{1}{2}}g_\lambda(T_\delta)(T_\delta+\lambda)^{\frac{1}{2}}\right\|_{\mathscr{B}(\mathcal{H}_t)}^2\cdot\left\|(T_\delta+\lambda)^{-\frac{1}{2}}(T_\nu+\lambda)^{\frac{1}{2}}\right\|_{\mathscr{B}(\mathcal{H}_t)}^2\label{Decomposition of Estimation Error: Third Step; Term a}\\
		&\quad\cdot\left\|(T_\nu+\lambda)^{-\frac{1}{2}}\left(g_D-T_\delta f_{P,\lambda}\right)\right\|_{\mathcal{H}_t}^2\label{Decomposition of Estimation Error: Third Step; Term b}.
	\end{align}
\end{subequations}
In addressing the first two terms in \eqref{Decomposition of Estimation Error: Third Step; Term a}, we will leverage the properties \eqref{Property of Regularization Family} of the regularization family and the upper bound \eqref{Upper Bound: Second Term in Term a of First Step Decomposition of Estimation Error}. These enable us to establish the following estimates, given by
\begin{align}
	\left\|(T_\delta+\lambda)^{\frac{1}{2}}g_\lambda(T_\delta)(T_\delta+\lambda)^{\frac{1}{2}}\right\|_{\mathscr{B}(\mathcal{H}_t)}^2
	&=\left\|(T_\delta+\lambda)g_\lambda(T_\delta)\right\|_{\mathscr{B}(\mathcal{H}_t)}^2
	\nonumber\\
	&\leq\left(\sup\limits_{0\leq t\leq\kappa^2}(t+\lambda)g_\lambda(t)\right)^2\leq4\label{Upper Bound: First Term in Term a of Third Step Decomposition of Estimation Error}
\end{align}
and
\begin{equation}\label{Upper Bound: Second Term in Term a of Third Step Decomposition of Estimation Error}
	\left\|(T_\delta+\lambda)^{-\frac{1}{2}}(T_\nu+\lambda)^{\frac{1}{2}}\right\|_{\mathscr{B}(\mathcal{H}_t)}^2
	=\left\|(T_\nu+\lambda)^{\frac{1}{2}}(T_\delta+\lambda)^{-1}(T_\nu+\lambda)^{\frac{1}{2}}\right\|_{\mathscr{B}(\mathcal{H}_t)}
	\leq3.
\end{equation}
Moving on, we address the final term in \eqref{Decomposition of Estimation Error: Third Step; Term b}. To establish an upper bound for this term, we will rely on Lemma \ref{Lemma: Empirical Error Lemma, Second One} that offers control over the error resulting from sampling. This kind of lemma is commonly found in statistical literature and is typically derived through the application of concentration inequalities. This lemma comes specifically from \cite{fischer2020sobolev}.

\begin{lemma}\label{Lemma: Empirical Error Lemma, Second One}
	Denote
	\begin{equation*}
	L_\lambda=\max\{L, \|f_{P,\lambda}-f^*\|_{L^\infty(\nu)}\}.
	\end{equation*}
	Then, for all $0<\alpha\leq1$, $\lambda>0$, $\tau\geq1$, and $n\geq1$, the following operator norm bound holds with probability not less than $1-2e^{-\tau}$:
	\begin{equation*}
		\begin{aligned}
			&\left\|(T_\nu+\lambda)^{-\frac{1}{2}}\left((g_D-T_\delta f_{P,\lambda})-(g_P-T_\nu f_{P,\lambda})\right)\right\|_{\mathcal{H}_t}^2\\
			&\leq\frac{64\tau^2}{n}\left(
			\sigma^2N_\nu(\lambda)+A_\alpha^2\frac{\|f^*-f_{P,\lambda}\|^2_{L^2(\nu)}}{\lambda^\alpha}+2A_\alpha^2\frac{L^2_\lambda}{n\lambda^\alpha}\right).
		\end{aligned}
	\end{equation*}
\end{lemma}

Now, we will present our upper bound for the final term in \eqref{Decomposition of Estimation Error: Third Step; Term b} as detailed in the following Lemma \ref{Lemma: Upper Bound of Term b of Third Step Decomposition of Estimation Error}.

\begin{lemma}\label{Lemma: Upper Bound of Term b of Third Step Decomposition of Estimation Error}
	For all $0<\alpha\leq1$, $0<\beta\leq2\xi$, $\lambda>0$, $\tau\geq1$, and $n\geq1$,  the following operator norm bound holds with probability not less than $1-2e^{-\tau}$:
	\begin{align}
		&\left\|(T_\nu+\lambda)^{-\frac{1}{2}}\left(g_D-T_\delta f_{P,\lambda}\right)\right\|_{\mathcal{H}_t}^2\nonumber\\
		&\leq\frac{128\tau^2}{n}\left(
		\sigma^2N_\nu(\lambda)+A_\alpha^2\|f^*\|_\beta^2\lambda^{\beta-\alpha}+2A_\alpha^2\frac{L^2_\lambda}{n\lambda^\alpha}\right)+2\|f^*\|_\beta^2\lambda^\beta\label{Upper Bound of Term b of Third Step Decomposition of Estimation Error}.
	\end{align}
\end{lemma}

\begin{proof}
	Initially, we establish
	\begin{subequations}
		\begin{align}
			&\left\|(T_\nu+\lambda)^{-\frac{1}{2}}\left(g_D-T_\delta f_{P,\lambda}\right)\right\|_{\mathcal{H}_t}^2\nonumber\\
			&=\left\|(T_\nu+\lambda)^{-\frac{1}{2}}\left((g_D-T_\delta f_{P,\lambda})-(g_P-T_\nu f_{P,\lambda})\right)+(T_\nu+\lambda)^{-\frac{1}{2}}(g_P-T_\nu f_{P,\lambda})\right\|_{\mathcal{H}_t}^2\nonumber\\
			&\leq2\left\|(T_\nu+\lambda)^{-\frac{1}{2}}\left((g_D-T_\delta f_{P,\lambda})-(g_P-T_\nu f_{P,\lambda})\right)\right\|_{\mathcal{H}_t}^2
			\label{Decomposition of Estimation Error: Fourth Step; Term a}\\
			&\quad+2\left\|(T_\nu+\lambda)^{-\frac{1}{2}}(g_P-T_\nu f_{P,\lambda})\right\|_{\mathcal{H}_t}^2\label{Decomposition of Estimation Error: Fourth Step; Term b}.
		\end{align}
	\end{subequations}
	For \eqref{Decomposition of Estimation Error: Fourth Step; Term a}, by applying Lemma \ref{Lemma: Empirical Error Lemma, Second One}, we can derive
	\begin{align}
		&\left\|(T_\nu+\lambda)^{-\frac{1}{2}}\left((g_D-T_\delta f_{P,\lambda})-(g_P-T_\nu f_{P,\lambda})\right)\right\|_{\mathcal{H}_t}^2
		\nonumber\\
		&\leq\frac{64\tau^2}{n}\left(
		\sigma^2N_\nu(\lambda)+A_\alpha^2\frac{\|f^*-f_{P,\lambda}\|^2_{L^2(\nu)}}{\lambda^\alpha}+2A_\alpha^2\frac{L^2_\lambda}{n\lambda^\alpha}\right)\label{Upper Bound of Term a of Fourth Step Decomposition of Estimation Error}.
	\end{align}
	This result is obtained with a probability of not less than $1-2e^{-\tau}$.
	When considering \eqref{Decomposition of Estimation Error: Fourth Step; Term b}, noting that $g_P=I_\nu^*f^*$ and $T_\nu=I_\nu^*I_\nu$, we arrive at
	\begin{equation*}
		\begin{aligned}
			\left\|(T_\nu+\lambda)^{-\frac{1}{2}}(g_P-T_\nu f_{P,\lambda})\right\|_{\mathcal{H}_t}^2
			&=\left\|(T_\nu+\lambda)^{-\frac{1}{2}}(I_\nu^*f^*-I_\nu^*I_\nu f_{P,\lambda})\right\|_{\mathcal{H}_t}^2\\
			&=\left\|(T_\nu+\lambda)^{-\frac{1}{2}}I_\nu^*\left(f^*-f_{P,\lambda}\right)\right\|_{\mathcal{H}_t}^2.\\
		\end{aligned}
	\end{equation*}
	On one hand, we have 
	\begin{equation*}	\left\|(T_\nu+\lambda)^{-\frac{1}{2}}I_\nu^*\right\|^2_{\mathscr{B}(L^2(\nu),\mathcal{H}_t)}=\sup\limits_{k\geq0}\frac{pe^{-t\lambda_k}}{\lambda+pe^{-t\lambda_k}}\leq1.
	\end{equation*}
	On the other hand, by integrating our previous approximation error result from Lemma \ref{Lemma: Gamma-Norm Upper Bound of Approximation Error} with $\gamma=0$, we obtain
	\begin{equation}\label{L-2 error bound in proof}
		\|f^*-f_{P,\lambda}\|_{L^2(\nu)}^2\leq\|f^*\|^2_\beta\lambda^{\beta}.
	\end{equation}
	From these analyses, we can conclude an upper bound for \eqref{Decomposition of Estimation Error: Fourth Step; Term b} as
	\begin{equation}\label{Upper Bound of Term b of Fourth Step Decomposition of Estimation Error}
		\left\|(T_\nu+\lambda)^{-\frac{1}{2}}(g_P-T_\nu f_{P,\lambda})\right\|_{\mathcal{H}_t}^2\leq\|f^*\|^2_\beta\lambda^{\beta}.
	\end{equation}
	By substituting \eqref{L-2 error bound in proof} back into \eqref{Upper Bound of Term a of Fourth Step Decomposition of Estimation Error} and integrating it with \eqref{Upper Bound of Term b of Fourth Step Decomposition of Estimation Error}, we are able to achieve the desired result and complete the proof. 
\end{proof}

Finally, by merging \eqref{Upper Bound: First Term in Term a of First Step Decomposition of Estimation Error}, \eqref{Upper Bound: Second Term in Term a of First Step Decomposition of Estimation Error}, \eqref{Upper Bound: First Term in Second Step Decomposition of Estimation Error}, \eqref{Upper Bound: First Term in Term a of Third Step Decomposition of Estimation Error}, \eqref{Upper Bound: Second Term in Term a of Third Step Decomposition of Estimation Error}, and \eqref{Upper Bound of Term b of Third Step Decomposition of Estimation Error} we are able to summarize the estimation error bound as follows.

\begin{lemma}\label{Lemma: Upper Bound of Estimation Error}
	Suppose that $ n\geq8A_\alpha^2\tau p_\lambda\lambda^{-\alpha} $. Then, for all $0<\alpha\leq1$, $0\leq\gamma\leq1$, $\gamma<\beta\leq1$, $\xi\geq\frac{1}{2}$, $\tau\geq1$, and $\lambda>0$, with probability not less than $1-4e^{-\tau}$, the following estimation error bound holds:
	\begin{equation}\label{Upper Bound of Estimation Error}
		\begin{aligned}
			\|f_{P,\lambda}-f_{D,\lambda}\|_\gamma^2
			&\leq\frac{9216\tau^2}{n\lambda^\gamma}\left(\sigma^2N_\nu(\lambda)+A_\alpha^2\|f^*\|_\beta^2\lambda^{\beta-\alpha}+2A_\alpha^2\frac{L_\lambda^2}{n\lambda^\alpha}\right)\\
			&\quad+168\|f^*\|_\beta^2\lambda^{\beta-\gamma}.
		\end{aligned}
	\end{equation}
\end{lemma}

\subsection{Proof of Theorem \ref{Theorem: Upper Bound}}\label{Subsection: Proof of Theorem1}

At this juncture, we are fully prepared to prove Theorem \ref{Theorem: Upper Bound}.

\noindent
\emph{Proof of Theorem \ref{Theorem: Upper Bound}}.
	Given the choice of $\lambda_n$ as
	\begin{equation*}
		\lambda_n\sim\left(\frac{(\log n)^\frac{m}{2}}{n}
		\right)^\frac{1}{\beta},
	\end{equation*}
	we first consider the condition on $ n $ proposed in Lemma \ref{Lemma: Upper Bound of Estimation Error}, which relates specifically to the derivation of the sample lower bound. Remember that from our notations, we have
	\begin{equation*}
		p_\lambda = \log\left(2eN_\nu(\lambda)\frac{\lambda+\|T_\nu\|}{\|T_\nu\|}\right).
	\end{equation*}
	Therefore, by applying Lemma \ref{Lemma: Upper Bound of Effective Dimension}, we can derive that there exists an $ n_0 $ such that, for all $ n\geq n_0 $, the following inequalities hold true:
	\begin{equation*}
		\begin{aligned}
			8A_\alpha^2\tau p_{\lambda_n}\lambda_n^{-\alpha}
			& = 8A_\alpha^2\tau\log\left(2eN_\nu(\lambda_n)\frac{\lambda_n+\|T_\nu\|}{\|T_\nu\|}\right)\lambda_n^{-\alpha}\\
			&\leq 8A_\alpha^2\tau\log\left(2.2eD(\log\lambda_n^{-1})^{\frac{m}{2}}\right)\lambda_n^{-\alpha}\\
			&\leq 16A_\alpha^2\tau\log\left((\log\lambda_n^{-1})^{\frac{m}{2}}\right)\lambda_n^{-\alpha}\\
			&= 8mA_\alpha^2\tau\log(\log\lambda_n^{-1})\lambda_n^{-\alpha}.\\
		\end{aligned}
	\end{equation*}
	Notice that
	\begin{equation*}
		\log(\log\lambda_n^{-1})\lambda_n^{-\alpha}\sim\log\left(\frac{1}{\beta}(\log n-\frac{m}{2}\log\log n)\right)(\log n)^{-\frac{\alpha m}{2\beta}}\cdot n^{\frac{\alpha}{\beta}}\lesssim n^{\frac{\alpha}{\beta}}.
	\end{equation*}
	By fixing an $\alpha$ that satisfies condition $ 0<\alpha<\beta $, there exists another $ n_1\geq n_0 $ such that, for all $ n\geq n_1 $, it holds
	\begin{equation*}
		n > 8mA_\alpha^2\tau\log(\log\lambda_n^{-1})\lambda_n^{-\alpha}.
	\end{equation*}
	Therefore, we can utilize \eqref{Upper Bound of Estimation Error} for $ n\geq n_1 $ to derive
	\begin{align}
		\|f_{P,\lambda_n}-f_{D,\lambda_n}\|_\gamma^2
		&\leq\frac{9216\tau^2}{n\lambda_n^\gamma}\left(\sigma^2N_\nu(\lambda_n)+A_\alpha^2\|f^*\|_\beta^2\lambda_n^{\beta-\alpha}+2A_\alpha^2\frac{L_\lambda^2}{n\lambda_n^\alpha}\right)\label{Estimation Error Bound; First Term}\\
		&\quad+168\|f^*\|_\beta^2\lambda_n^{\beta-\gamma}.\nonumber
	\end{align}
	Moreover, the sample lower bound $ n_1 $ (and $ n_0 $) only depends on constants $ m,D,\|T_\nu\| $, and parameters $ \alpha,\beta,\tau $. 

	Without any loss of generality, we assume that $\lambda_n\leq0.1$ for $ n\geq n_1 $. Regarding the third term in the bracket of \eqref{Estimation Error Bound; First Term}, combining \eqref{L-infty Bound of Approximation Error} with Assumption \ref{Assumption: Source Condition} we obtain
	\begin{equation*}
		\begin{aligned}
			A_\alpha^2\frac{L_{\lambda_n}^2}{n\lambda_n^\alpha}
			&=\frac{A_\alpha^2}{n\lambda_n^\alpha}\max\{L^2, \|f_{P,\lambda_n}-f^*\|^2_{L^\infty(\nu)}\}\\
			&\leq\frac{A_\alpha^2}{n\lambda_n^\alpha}\max\{L^2, A^2_\alpha R^2\lambda_n^{\beta-\alpha}\}\leq \frac{K_1}{n\lambda_n^\alpha}
		\end{aligned}
	\end{equation*}
	with $K_1=\max\{A_\alpha^2L^2,A_\alpha^4R^2\}$.	
	Taking use of Lemma \ref{Lemma: Upper Bound of Effective Dimension} and Assumption \ref{Assumption: Source Condition}, we can derive that 
	\begin{equation}\label{Estimation error bound; second bracket}
		\|f_{P,\lambda_n}-f_{D,\lambda_n}\|_\gamma^2
		\leq\frac{9216\tau^2}{n\lambda_n^\gamma}\left(D\sigma^2(\log\lambda_n^{-1})^\frac{m}{2}+A_\alpha^2R^2\lambda_n^{\beta-\alpha}+\frac{2K_1}{n\lambda_n^\alpha}\right)+168R^2\lambda_n^{\beta-\gamma}
	\end{equation}
	Now we turn our attention to those terms in the bracket of \eqref{Estimation error bound; second bracket}. Recalling the choice of $\lambda_n$, we have
	\begin{equation*}
		\begin{aligned}
			(\log\lambda_n^{-1})^\frac{m}{2}+\lambda_n^{\beta-\alpha}+\frac{1}{n\lambda_n^\alpha}
			&\sim\left(\frac{1}{\beta}(\log n-\frac{m}{2}\log\log n)\right)^{\frac{m}{2}}+(\log n)^{\frac{m}{2}(1-\frac{\alpha}{\beta})}n^{\frac{\alpha}{\beta}-1}\\
			&\quad+(\log n)^{-\frac{\alpha m}{2\beta}}n^{\frac{\alpha}{\beta}-1}\\
			&\lesssim\left(\frac{1}{\beta}(\log n-\frac{m}{2}\log\log n)\right)^{\frac{m}{2}}\sim(\log\lambda_n^{-1})^\frac{m}{2}.
		\end{aligned}
	\end{equation*}
	Therefore, we can choose $ K_2=27648\max\{D\sigma^2,A_\alpha^2R^2,2K_1\} $ such that, for all $ n\geq n_1 $, we have
	\begin{equation}\label{Estimation Error Bound; Final Version}
		\|f_{P,\lambda_n}-f_{D,\lambda_n}\|_\gamma^2
		\leq K_2\tau^2\frac{(\log\lambda_n^{-1})^\frac{m}{2}}{n\lambda_n^\gamma}+168R^2\lambda_n^{\beta-\gamma}
	\end{equation}

	Finally, by combining \eqref{Estimation Error Bound; Final Version} with the approximation error bound \eqref{Gamma-Norm Upper Bound of Approximation Error} derived from Lemma \ref{Lemma: Gamma-Norm Upper Bound of Approximation Error}, we arrive at our targeted $\gamma$-norm error bound as follows
	\begin{equation*}
		\begin{aligned}
			\|f_{D,\lambda_n}-f^*\|_\gamma^2
			&\leq K_2\tau^2\frac{(\log\lambda_n^{-1})^\frac{m}{2}}{n\lambda_n^\gamma}+169R^2\lambda_n^{\beta-\gamma}\\
			&\leq\tau^2\lambda_n^{\beta-\gamma}\left(K_2\frac{(\log\lambda_n^{-1})^\frac{m}{2}}{n\lambda_n^{\beta}}+169R^2\right).
		\end{aligned}
	\end{equation*}
	A direct computation indicates that
	\begin{equation*}
		\frac{(\log\lambda_n^{-1})^\frac{m}{2}}{n\lambda_n^{\beta}}
		\sim\frac{\left(\frac{1}{\beta}(\log n-\frac{m}{2}\log\log n)\right)^\frac{m}{2}}{(\log n)^\frac{m}{2}}\leq\left(\frac{1}{\beta}\right)^\frac{m}{2}.
	\end{equation*}
	Consequently, we finally obtain an upper bound as
	\begin{equation*}
		\|f_{D,\lambda_n}-f^*\|_\gamma^2
		\lesssim \tau^2\lambda_n^{\beta-\gamma}
		\leq C\tau^2\left(\frac{(\log n)^\frac{m}{2}}{n}\right)^\frac{\beta-\gamma}{\beta},
	\end{equation*}
	which valids for all $ n\geq n_1 $. Here, the constant $C$ only depends on constants $ m,D,\sigma,R,L $, and parameters $\alpha,\beta $. The proof is then completed.
\qed

\section{Lower Bound Analysis}\label{Section: Lower Bound Analysis}

In this section, we follow the approach presented in \cite{DBLP:books/daglib/0035708} to derive the minimax lower bound. This methodology involves leveraging Lemma \ref{Tsybakov's Lemma} from \cite{DBLP:books/daglib/0035708}, which is instrumental in our analysis. In our notations, the Kullback–Leibler divergence between two probability measures $\rho_1$ and $\rho_2$ on $(\Omega, \mathcal{F})$, denoted by $D_{KL}(\rho_1 \parallel \rho_2)$, is defined as
\begin{equation*}
	D_{KL}(\rho_1 \parallel \rho_2) = \int_{\Omega} \log\left(\frac{\mathrm{d}\rho_1}{\mathrm{d}\rho_2}\right) \mathrm{d}\rho_1,
\end{equation*}
provided that $\rho_1$ is absolutely continuous with respect to $\rho_2$, and $ D_{KL}(\rho_1 \parallel \rho_2) = \infty$ otherwise.

\begin{lemma}\label{Tsybakov's Lemma}
	Suppose that there is a non-parametric class of functions $\Theta$ and a semi-distance $d(\cdot,\cdot)$ on $\Theta$, $\{P_\theta:\theta\in\Theta\}$ is a family of probability distributions indexed by $\Theta$. Assume that $K\geq2$ and $\Theta$ contains elements $\theta_0,\theta_1,\cdots,\theta_K$ such that
	\begin{itemize}
	\item[(1)] $\forall\;0\leq i<j\leq K$,
		\begin{equation}\label{Seperation Condition}
			d(\theta_i,\theta_j)\geq2s>0.
		\end{equation}
	\item[(2)] $P_j\ll P_0,\;\forall j=1,\dots,K$ and
		\begin{equation}\label{KL-Divergence Condition}
			\frac{1}{K}\sum_{j=1}^KD_{KL}(P_j||P_0)\leq a\log K
		\end{equation}
		with $0<a<\frac{1}{8}$ and $P_j=P_{\theta_j},\;j=0,1,\dots,K$.
	\end{itemize}
	Then it holds
	\begin{equation*}
		\inf\limits_{\hat{\theta}}\sup\limits_{\theta\in\Theta}P_\theta\left(d(\hat{\theta},\theta)\geq s\right)\geq\frac{\sqrt{K}}{1+\sqrt{K}}\left(1-2a-\sqrt{\frac{2a}{\log K}}\right).
	\end{equation*}
\end{lemma}

Our current objective is to construct a family of probability distributions with elements $P_0,P_1,\dots,P_K$, which adhere to the Kullback-Leibler divergence condition outlined in \eqref{KL-Divergence Condition}. Let $\bar{\sigma}=\min\{\sigma,L\}$ be given. For any measurable function $f:\mathcal{M}\to\mathbb{R}$ and for any $x\in\mathcal{M}$, we define the conditional distribution given $x$ as $P_f(\cdot|x)=\mathcal{N}(f(x),\bar{\sigma}^2)$. This distribution is a normal distribution with a mean of $f(x)$ and a variance of $\bar{\sigma}^2$.
Further, we construct a probability distribution $P_f$ related to $f$ on $\mathcal{M}\times\mathbb{R}$. This distribution is designed such that its marginal distribution on $\mathcal{M}$ is $\nu$, and its conditional distribution given $x\in\mathcal{M}$ on $\mathbb{R}$ is $P_f(\cdot|x)$. In practical terms, this means that we sample $x$ from $\nu$ and set $y=f(x)+\varepsilon$, where $\varepsilon\sim\mathcal{N}(0,\bar{\sigma}^2)$ represents an independent Gaussian random error. It can be straightforwardly verified that for all $f\in L^2(\nu)$, the condition $|P_f|^2_2=\|f\|_{L^2(\nu)}^2+\bar{\sigma}^2<\infty$ is satisfied. Furthermore, the conditional mean function of this distribution, which corresponds to the regression function $f^*$, is identical to $f$.
It is important to note that probability distributions of this nature inherently satisfy the moment condition \eqref{Moment Condition} outlined in Assumption \ref{Assumption: Moment Condition}, as specified in Lemma \ref{Lemma: Moment Condition of Gaussian}. This compliance is deduced from a straightforward computation involving the high-order moments of Gaussian variables. For a more detailed explanation of this calculation, we refer to \cite{fischer2020sobolev}. This construction lays the foundation for our analysis, ensuring that the distributions we work with are both theoretically sound and aligned with the assumptions underpinning our approach.

\begin{lemma}\label{Lemma: Moment Condition of Gaussian}
	For a measurable function $f:\mathcal{M}\to\mathbb{R}$ with its corresponding probability measure $P_f$ defined above, the moment condition \eqref{Moment Condition} is satisfied for $\sigma=L=\bar{\sigma}$.
\end{lemma}

Given that $f^*=f$ holds for the probability distribution $P_f$, we can focus on constructing a specific family of functions. This family, comprising elements $f_0,f_1,\dots,f_K$, must satisfy the source condition \eqref{Source Condition} as stipulated in Assumption \ref{Assumption: Source Condition}. To develop these function elements, we define
\begin{equation}\label{Definition of f_i}
	f_i = \epsilon^\frac{1}{2}\sum_{j=1}^k\omega_j^{(i)}p^{\frac{\gamma}{2}}e^{-\frac{\gamma}{2}\lambda_{k+j}t}f_{k+j}
\end{equation}
for $0\leq i\leq K$, where $\omega^{(i)}=(\omega_1^{(i)},\dots,\omega_k^{(i)})\in\{0,1\}^k$ are binary strings and $\epsilon,k$ are to be determined later. Our approach, which hinges on the use of binary strings, needs to fulfill the separation condition \eqref{Seperation Condition} as stated in Lemma \ref{Tsybakov's Lemma}. Ensuring that a sufficient number of binary strings $\omega$ exists, with significant separation between them, is crucial for our analysis.
To substantiate the presence of these adequately distanced binary strings, we will invoke Lemma \ref{Lemma: G-V Bound}, commonly referred to as the Gilbert-Varshamov Bound. This lemma is a well-known result in coding theory and provides a means to guarantee the existence of binary strings with the required properties. A detailed discussion of this lemma and its implications also can be found in \cite{DBLP:books/daglib/0035708}. This construction strategy is pivotal as it forms the basis of our function family, ensuring that it adheres to the necessary conditions and assumptions for our analysis to be valid.

\begin{lemma}\label{Lemma: G-V Bound}
	For $k\geq8$, there exists $K\geq2^\frac{k}{8}$ and binary strings $w^{(0)},\dots,\omega^{(K)}\in\{0,1\}^m$ with $\omega^{(0)}=(0,\dots,0)$, $\omega^{(j)}=(\omega_1^{(j)},\dots,\omega_k^{(j)})$ such that
	\begin{equation}\label{G-V Bound}
		\sum_{i=1}^k\left(\omega_i^{(l)}-\omega_i^{(j)}\right)^2\geq\frac{k}{8}
	\end{equation}
	for all $0\leq l\neq j\leq K$.
\end{lemma}

In order to meet the Kullback-Leibler divergence condition specified in \eqref{KL-Divergence Condition} as per Lemma \ref{Tsybakov's Lemma}, it's necessary to perform direct computations pertaining to the Gaussian-type probability distribution we have constructed. These calculations, as outlined in Lemma \ref{Lemma: KL-Divergence of Guassain}, are crucial for ensuring that the probability distribution aligns with the required divergence condition, which is a fundamental aspect of establishing the minimax lower bound. For a more in-depth understanding and additional context, one can refer to \cite{blanchard2018optimal}.

\begin{lemma}\label{Lemma: KL-Divergence of Guassain}
	For $f,f'\in L^2(\nu)$ and $n\geq1$, the corresponding probability distribution $P_f,P_{f'}$ defined above satisfy
	\begin{equation}\label{KL-Divergence of Guassain}
	D_{KL}(P_f^n||P_{f'}^n)=\frac{n}{2\bar{\sigma}^2}\|f-f'\|^2_{L^2(\nu)}.
	\end{equation}
\end{lemma}

Having completed all the preparatory works, we now proceed to furnish the proof for Theorem \ref{Theorem: Lower Bound}.

\noindent
\emph{Proof of Theorem \ref{Theorem: Lower Bound}}.
	Our strategic approach involves a careful selection of $\epsilon,k$, aimed at bounding the $\beta$-norm of our constructed $f_i$ from \eqref{Definition of f_i}. This is crucial to ensure that the source condition \eqref{Source Condition} is satisfied. Simultaneously, we will also focus on bounding the Kullback-Leibler divergence of the corresponding probability distributions to adhere to the condition \eqref{KL-Divergence Condition}.
	Once these bounds are established and both the source condition and KL-divergence requirements are met, we will then turn to Lemma \ref{Tsybakov's Lemma}. Utilizing this lemma, we will derive the desired minimax lower bound. This culmination of steps represents a comprehensive method, bringing together various facets of our analysis to ultimately reach the sought-after minimax lower bound result. This careful and methodical approach is essential for the robustness and validity of our proof of Theorem \ref{Theorem: Lower Bound}.

	Setting:
	\begin{equation*}
		f_i = \epsilon^\frac{1}{2}\sum_{j=1}^k\omega_j^{(i)}p^{\frac{\gamma}{2}}e^{-\frac{\gamma}{2}\lambda_{k+j}t}f_{k+j},\quad i=0,1,\dots,K
	\end{equation*}
	with binary strings $\omega^{(i)}$ and a certain $K\geq 2^\frac{k}{8}$ derived from Lemma \ref{Lemma: G-V Bound}, we define $\epsilon=C_{\epsilon}n^{-s}$ and $k=C_{k}(\log n)^\frac{m}{2}$ with constants $s,C_{\epsilon},C_k$ to be determined later. Given $\gamma<\beta$, and applying \eqref{Eigenvalue Growthness of Laplacian}, we establish
	\begin{align}
		\|f_i\|_\beta^2
		&=\epsilon\sum_{j=1}^k p^{\gamma-\beta}e^{(\beta-\gamma)\lambda_{k+j}t}\left(\omega_j^{(i)}\right)^2\nonumber\\
		&\leq\epsilon\sum_{j=1}^k p^{\gamma-\beta}e^{(\beta-\gamma)tC_{up}(k+j)^\frac{2}{m}}\nonumber\\
		&\leq\epsilon kp^{\gamma-\beta}e^{(\beta-\gamma)tC_{up}(2k)^\frac{2}{m}}\nonumber\\
		&=p^{\gamma-\beta}C_{\epsilon}C_kn^{-s}(\log n)^\frac{m}{2}n^{(\beta-\gamma)tC_{up}(2C_k)^\frac{2}{m}}\label{Upper Bound: Gamma Norm of f_i}.
	\end{align}
	To satisfy the source condition \eqref{Source Condition}, we initially need to set
	\begin{equation}\label{Inequality from Gamma Norm Upper Bound of f_i}
		(\beta-\gamma)tC_{up}(2C_k)^\frac{2}{m}-s<0.
	\end{equation}
	Then, utilizing \eqref{KL-Divergence of Guassain} from Lemma \ref{Lemma: KL-Divergence of Guassain} and combining it with \eqref{Eigenvalue Growthness of Laplacian}, we can determine
	\begin{equation*}
			\begin{aligned}
			D_{KL}(P_{f_i}^n||P_{f_0}^n)
			&=\frac{n}{2\bar{\sigma}^2}\|f_i\|^2_{L^2(\nu)}\\
			&=\frac{n\epsilon}{2\bar{\sigma}^2}\sum_{j=1}^k p^\gamma e^{-\gamma\lambda_{k+j}t}\left(\omega_j^{(i)}\right)^2\\
			&\leq\frac{n\epsilon}{2\bar{\sigma}^2}kp^\gamma e^{-\gamma tC_{low}k^\frac{2}{m}}.\\
		\end{aligned}
	\end{equation*}
	Recalling $K\geq2^\frac{k}{8}$ and fixing $a\in(0,\frac{1}{8})$, the Kullback-Leibler divergence condition \eqref{KL-Divergence Condition} results in
	\begin{equation*}
		D_{KL}(P_{f_i}^n||P_{f_0}^n)
		\leq\frac{n\epsilon}{2\bar{\sigma}^2}kp^\gamma e^{-\gamma tC_{low}k^\frac{2}{m}}
		\leq a\frac{\log2}{8}k
		\leq a\log K.
	\end{equation*}
	In essence, this implies
	\begin{equation}\label{Upper Bound: KL-Divergence of f_i}
		C_\epsilon n^{1-s}n^{-\gamma tC_{low}C_k^\frac{2}{m}}\leq ap^{-\gamma}\frac{\bar{\sigma}^2\log2}{4}.
	\end{equation}
	Thus, we should set
	\begin{equation}\label{Inequality form Upper Bound of KL-Divergence of f_i}
		(1-s)-\gamma tC_{low}C_k^\frac{2}{m}\leq0.
	\end{equation}
	When we merge \eqref{Inequality from Gamma Norm Upper Bound of f_i} with \eqref{Inequality form Upper Bound of KL-Divergence of f_i}, specifically for the case where $\gamma=0$, we can let $s=1$ and choose $C_k$ such that
	\begin{equation*}
		C_k^\frac{2}{m}<\frac{1}{\beta tC_{up}2^\frac{2}{m}}.
	\end{equation*}
	In the general scenario where $\gamma>0$, we are limited to setting
	\begin{equation*}
		s>\frac{(\beta-\gamma)tC_{up}2^\frac{2}{m}}{(\beta-\gamma)tC_{up}2^\frac{2}{m}
		+\gamma tC_{low}}
		\geq\frac{(\beta-\gamma)2^\frac{2}{m}}{(\beta-\gamma)2^\frac{2}{m}+\gamma}
	\end{equation*}
	and selecting $C_k$ so that
	\begin{equation*}
		\frac{1-s}{\gamma tC_{low}}\leq C_k^\frac{2}{m}<\frac{s}{(\beta-\gamma)tC_{up}2^\frac{2}{m}}.
	\end{equation*}
	Moreover, as our constructed $f_i$ needs to satisfy both \eqref{Upper Bound: Gamma Norm of f_i} and \eqref{Upper Bound: KL-Divergence of f_i}, we also need to pick a $C_\epsilon$ that fulfills
	\begin{equation*}
		C_\epsilon\leq \min\left\{ap^{-\gamma}\frac{\bar{\sigma}^2\log2}{4},p^{\beta-\gamma}\frac{R^2}{C_k}\right\}.
	\end{equation*}
	Then, for large $n$, both the source condition \eqref{Source Condition} and the KL-divergence condition \eqref{KL-Divergence Condition} are satisfied. 
	Now we consider the probability distribution family $\{P_{f}^n:\|f\|_\beta^2\leq R^2\}$ indexed by $f$, endowed with a distance defined by $d(f_i,f_j)=\|f_i-f_j\|_\gamma$. Applying the Gilbert-Varshamov bound \eqref{G-V Bound} to this family yields
	\begin{equation*}
		d(f_i,f_j)^2=\|f_i-f_j\|_\gamma^2=\epsilon\sum_{l=1}^k\left(\omega_l^{(i)}-\omega_l^{(j)}\right)^2\geq\epsilon\frac{k}{8}=Ca\frac{(\log n)^\frac{m}{2}}{n^s},
	\end{equation*}
	where the constant $C$ is independent of $n,a$. Finally, by applying Lemma \ref{Tsybakov's Lemma}, we obtain
	\begin{equation}\label{Minimax Lower Bound of Gamma-Norm}
		\inf\limits_{\hat{f}}\sup\limits_{f^*}P^n_{f^*}\left(\|\hat{f}-f^*\|^2_\gamma\geq C'a\frac{(\log n)^\frac{m}{2}}{n^s}\right)\geq\frac{\sqrt{K}}{1+\sqrt{K}}\left(1-2a-\sqrt{\frac{2a}{\log K}}\right)
	\end{equation}
	with $ C'=\frac{1}{4}C $. 
	For sufficiently large $n$ (and correspondingly large $K$), the right-hand side of \eqref{Minimax Lower Bound of Gamma-Norm} exceeds $1-3a$. Therefore, for all $\tau\in(0,1)$ and any estimator $\hat{f}$, we can identify a function $f^*$ meeting the source condition \eqref{Source Condition}. Moreover, its corresponding probability distribution $P_{f^*}$, as defined above, will satisfy
	\begin{equation*}
		\|\hat{f}-f^*\|^2_\gamma\geq C''\tau^2\frac{(\log n)^\frac{m}{2}}{n^s}
	\end{equation*}
	with a probability of at least $1-\tau^2$. In this context, $ C''=\frac{1}{12}C $. This confirmation substantiates our proposed result. The proof is finished.
\qed

\bibliographystyle{plain}
\bibliography{Reference}
\end{document}